\documentclass[12pt]{article}
\usepackage[sort&compress]{natbib}
\usepackage[utf8x]{inputenc}
\usepackage[english]{babel}
\usepackage{graphicx, amsfonts,amsmath,amssymb,amsthm}
\usepackage{commands}
\usepackage{tikzsettings}
\usepackage{longtable,enumitem}
\usepackage{fullpage}
\usepackage{wrapfig}
\usepackage{array}
\usepackage{breakcites}
\date{\today} 

\title{ 
    Learning structured approximations of combinatorial optimization problems.
}
\author{Axel Parmentier \\ CERMICS, Ecole des Ponts, Marne-la-Vallée, France \\ \normalsize{\url{axel.parmentier@enpc.fr}}}
\date{\today}

\begin{document} 

\maketitle

\begin{abstract}
    Machine learning pipelines that include a combinatorial optimization layer can give surprisingly efficient heuristics for difficult combinatorial optimization problems.
    Three questions remain open: 
    which architecture should be used, 
    how should the parameters of the machine learning model be learned,
    and what performance guarantees can we expect from the resulting algorithms?
    Following the intuitions of geometric deep learning, we explain why equivariant layers should be used when designing such pipelines, and illustrate how to build such layers on routing,  scheduling, and  network design applications.
    We introduce a learning approach that enables to learn such pipelines when the training set contains only instances of the difficult optimization problem and not their optimal solutions, and show its numerical performance on our three applications. 
    Finally, using tools from statistical learning theory, we prove a theorem showing the convergence speed of the estimator. 
    As a corollary, we obtain that, if an approximation algorithm can be encoded by the pipeline for some parametrization, then the learned pipeline will retain the approximation ratio guarantee. 
    On our network design problem, our machine learning pipeline has the approximation ratio guarantee of the best approximation algorithm known and the numerical efficiency of the best heuristic.

    \end{abstract}

\section{Introduction}

In the last few years, more and more attention have been given to the construction of machine learning algorithms which, given an input $x$, can predict an output $y$ in a combinatorially large set $\calY(x)$. 
An approach that is getting more popular to address this problem consists in embedding a combinatorial optimization (CO) layer in a machine learning pipeline.
As illustrated on Figure~\ref{fig:pipeline}, the resulting pipeline typically chains a statistical model, a combinatorial optimization problem, and possibly a post-processing algorithm.

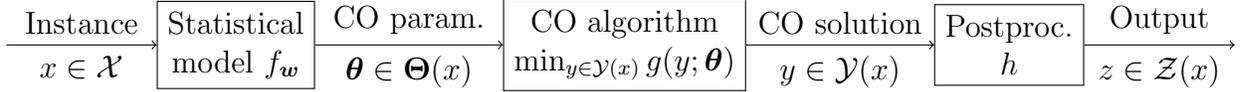
\begin{figure}
    \begin{tikzpicture}
        \node (o) {};
        \node[draw,right=2cm of o,align=center] (m) {Statistical \\ model $f_{\bfw}$};
        \node[draw,right=2.5cm of m,align=center] (co) {CO algorithm \\  $\min_{y \in \calY(x)} g(y;\bftheta)$};
        \node[draw,right=2.5cm of co, align=center] (pp) {Postproc. \\ $h$};
        \node[right=2cm of pp] (d) {};

        \draw[->] (o) to node[above]{Instance} node[below]{$x \in \calX$} (m);
        \draw[->] (m) to node[above]{CO param.} node[below]{$\bftheta \in \bfTheta(x)$} (co);
        \draw[->] (co) to node[above]{CO solution} node[below]{$y \in \calY(x)$} (pp);
        \draw[->] (pp) to node[above]{Output} node[below]{$z \in\calZ(x)$} (d);

    \end{tikzpicture}
    \caption{A machine learning pipeline with a combinatorial optimization (CO) layer}
    \label{fig:pipeline}
\end{figure}

Such pipelines can be used as heuristics for difficult combinatorial optimization problems.
Let us consider a combinatorial optimization problem of interest
\begin{equation}\label{eq:hardProblem}\tag{Pb}
    \min_{z \in \calZ(x)} C(z,x).
\end{equation}
Here $x$ is an instance in a \emph{set of instance} $\calX$, and $\calZ(x)$ denotes the set of feasible solutions of $x$. 
Contrary to what is usual in combinatorial optimization, we include the instance $x$ in the objective function $c(z;x)$. 

When we use our machine learning pipeline to solve~\eqref{eq:hardProblem}, we use the statistical model $f_{\bfw}$ to obtain the parameter $\bftheta$ of the auxiliary combinatorial optimization problem 
\begin{equation}\label{eq:easyProblem}\tag{CO-layer}
    \min_{y \in \calY(x)} g(y,\bftheta),
\end{equation}
and then decode the solution $y$ of this problem into a solution $z \in \calZ(x)$ of the initial problem.
Such a pipeline is useful when we have much more efficient algorithms for our combinatorial optimization layer problem~\eqref{eq:easyProblem} than for the problem of interest~\eqref{eq:hardProblem}. 

Since the preprocessing $h$ is assumed deterministic, it will not play a major role on the learning algorithm. It will therefore be convenient to bring back the cost on $\calY(x)$. For $y$ in $\calY(x)$, we define
$$ c(y,x) = C(h(y),x). $$

\paragraph*{Running example: Two stage spanning tree}
Let $G = (V,E)$ be an undirected graph, and $S$ be a finite set of scenarios.
% For each edge $e$ in $E$, we have a first stage cost $c_e$, and a second stage cost $d_{es}$ which depends on the scenario $s$ in $S$.
The objective is to build a spanning tree on $G$ of maximum cost on a two stage horizon. 
Building edge $e$ in the first stage costs $c_e \in \bbE$, while building it in the second stage under scenario $s$ costs $d_{es} \in \bbR$.
The decision maker does not know the scenario $s$ when it chooses which first stage edges to build.
Denoting $\calT$ the set of spanning trees, we can formulate the problem as
\begin{equation}
    \label{eq:twoStageSpanningTreeProblem}
    \min \Big\{\sum_{e \in E_1} c_e + \frac{1}{|S|}\sum_{e \in E_s}d_{es} \colon E_1 \cap E_s = \emptyset \text{ and } (V,E_1\cup E_s) \in \calT \text{ for all $s$ in $S$} \Big\}.
\end{equation}
% \begin{subequations}
%     \begin{alignat}{2}
%         \min_{\bfx,\bfy}\, & \bfc^{\top}\bfx + \frac{1}{|S|} \sum_{s \in S} \bfd_s^{\top}\bfy_s \\
%         \mathrm{s.t.} \,& \bfx + \bfy_s \in \calT &\quad& \text{for all $s$ in $S$} \\
%         & \bfy_s \in \{0,1\}^E && \text{for all $s$ in $S$} \\
%         & \bfx \in \{0,1\}^E 
%     \end{alignat}
% \end{subequations}
% where $\bfx = (x_e)_{e \in E}$ and $\bfy_s = (y_{es})_{e \in E}$ are vectors of binary variables, with $x_e =1$ if and only if $e$ is in the first stage solution, and $y_{es} = 1$ if and only if $e$ is in the second stage solution under scenario $s$. 
When we restrict ourselves to $c_e \leq 0$ and $d_{es} \leq 0$, we obtain the two stage maximum weight spanning tree. 
\citet{escoffierTwostageStochasticMatching2010} show that this restriction is APX-complete, and introduce a 2-approximation algorithm for the maximization problem, which translates into a 1/2-approximation algorithm for the minimization problem.

% \paragraph{Easy problem: } 
\paragraph{Running example pipeline}
Remark that an optimal solution of the single scenario version of the problem
\begin{equation}
    \label{eq:twoStageSpanningTreeProblem_easy}
    \min \Big\{\sum_{e \in E_1} \bar c_e + \frac{1}{|S|}\sum_{e \in E_2} \bar d_e \colon E_1 \cap E_2 = \emptyset \text{ and } (V,E_1\cup E_2) \in \calT \Big\}
\end{equation}
% \begin{subequations}
%     \label{eq:twoStageSpanningTreeProblem_easy}
%     \begin{alignat}{2}
%         \min_{\bfx,\bfy}\, & \bar\bfc^{\top}\bfx + \bar\bfd^{\top}\bfy \\
%         \mathrm{s.t.} \,& \bfx + \bfy \in \calT  \\
%         & \bfx,\bfy \in \{0,1\}^E 
%     \end{alignat}
% \end{subequations}
is a minimum weight spanning tree on $G$ with edge weights $\min(\bar c_e, \bar d_e)$. It can therefore be easily solved using Kruskal's algorithm, and we therefore suggest using~\eqref{eq:twoStageSpanningTreeProblem_easy} as combinatorial optimization layer~\eqref{eq:easyProblem}. 
Hence, we have $\bftheta = (c_e,d_e)_{e \in E}$ and $\bfTheta(x) = \bbR^{2E}$.

Our decoder $h$ rebuilds a solution $z$ of~\eqref{eq:twoStageSpanningTreeProblem} from a solution $y =(\bar E_1,\bar E_2)$ of~\eqref{eq:twoStageSpanningTreeProblem_easy}. It relies on the following result.
Given a forest $F$, Kruskal's algorithm can be adapted to find a minimum weight spanning tree containing $F$.
We take $\bar E_1$ as the first stage solution of~\eqref{eq:twoStageSpanningTreeProblem}, and use the variant of Kruskal's algorithm with edge weights $d_s$ to rebuild the $E_s$.
We then compare this solution to the optimal solution $E_1 = \emptyset$ and return the best of the two as $z$.
% We therefore build a candidate solution $(E_1,(E_s){s \in S})$ by taking $E_1 = $
% Our decoder $h$ works as follows: given the solution $y =(E_1,E_2)$

\paragraph{Structure of the combinatorial optimization layer.} When building a solution pipeline for a combinatorial algorithm, we typically want our pipeline to be able to address instances of very different size: Instances of our running example may have $20$ or $1000$ edges.
It means that the graph $G$ used in the combinatorial optimization layer~\eqref{eq:twoStageSpanningTreeProblem_easy} depends on the instance $x$ of~\eqref{eq:twoStageSpanningTreeProblem}, and hence the parameter $\bftheta$ belongs to the set $\bbR^{2E}$ which also depends on $x$.
This is the reason why, in our pipeline, the set of solutions $\calY(x)$ and the parameter space $\bfTheta(x)$ both depend on $x$.
On the contrary, since we want to use the same model and hence the same $f_{\bfw}$ on different instances, the space $\calW$ does not depend on $x$.
This raises the question of \emph{how to build statistical model $f_{\bfw}$ whose output dimension depends on the input dimension}.
More generally, such an approach can work only if~\eqref{eq:easyProblem} retains most of the ``structure'' of~\eqref{eq:hardProblem}.

\paragraph{Learning algorithm.}
Finally, the purpose of the learning algorithm is to find a parameter $w \in \calW$ such that the pipeline outputs a good solution of~\eqref{eq:hardProblem}. 
Approaches in the literature typically use a \emph{learning by imitation approach}, with a training set $(x_1,z_1),\ldots,(x_n,z_n)$ containing instances of~\eqref{eq:hardProblem} and their hard problem solution.
A drawback of such an approach is that it requires another solution algorithm for~\eqref{eq:hardProblem} to compute the $z_i$.
In this paper, we focus on the \emph{learning by experience} setting where the training set contains only instances $x_1,\ldots,x_n$.

\paragraph{Related works.}
% Combinatorial optimization layers in deep learning are an active field of research.
The interactions between combinatorial optimization and machine learning is an active research area~\citep{bengioMachineLearningCombinatorial2021}. Combinatorial optimization layers in deep learning belong to the subarea of end-to-end learning methods for combinatorial optimization problems recently surveyed by~\citet{kotaryEndtoEndConstrainedOptimization2021}. 
This field can be broadly classified in two subfields.
Machine learning augmented combinatorial optimization uses machine learning to take heuristic decisions within combinatorial optimization algorithms.
We survey here combinatorial optimization augmented machine learning, which inserts combinatorial optimization oracles within machine learning pipelines.

Structured learning approaches were the first to introduce these methods in the early 2000s~\citep{nowozinStructuredLearningPrediction2010} in the machine learning community.
They mainly considered maximum a posteriori problems in probabilistic graphical models as combinatorial optimization layers, with applications to computer vision, and sorting algorithms with applications to ranking.
They were generally trained using the structured Hinge loss or a maximum likelihood estimator.
A renewed interest for optimization layers in deep learning pipeline has emerged in the last few years has emerged in the machine learning community, and notably continuous optimization layers~\citep{amosOptNetDifferentiableOptimization2017,blondelEfficientModularImplicit2022}. 
% These approach have been considered mostly in the machine learning community and 
Remark that these pipelines are generally trained using a learning by imitation paradigm.

We focus here on combinatorial optimization layers. 
Among these, linear optimization layers have received the most attention. %, mostly in a learning by imitation paradigm.
Two challenges must be addressed.
First, since the mapping that associated to the objective parameter vector $\bftheta$ the output $y$ is piecewise constant, and deep learning networks are generally trained using stochastic gradient descent, meaningful approximations of the must be proposed gradient~\citep{vlastelicaDifferentiationBlackboxCombinatorial2020}.
Second a loss quantifying the error between its target must be proposed.
\citet{blondelLearningFenchelYoungLosses2020} address these challenges with an elegant solution based on convex duality: the linear objective is regularized with a convex penalization, which leads to meaningful gradients. Fenchel Young inequality in convex duality then gives a natural definition of the loss function.
\citet{berthetLearningDifferentiablePerturbed2020} have shown that this approach can be extended to the case where a random perturbation is added to the objective instead of a convex regularization.
When it comes to integer linear programs, \citet{mandi2020smart} suggest using the linear relaxation during the learning phase.

The author recently introduced the idea of building heuristics for hard combinatorial optimization problems with pipelines with combinatorial optimization layers~\citep{parmentierLearningApproximateIndustrial2019}.
The closest contribution to our learning by experience setting is the smart predict then optimize method of \citet{elmachtoubSmartPredictThen2021}. It considers the case where there is no decoder and the cost function $c(y,x) = \theta^* x$ is actually the linear objective of the combinatorial optimization layer~$g(y,\bftheta) = \bftheta y$ for an unknown true parameter $\theta^*$. They propose a generalization of the structured Hinge loss to that setting. 

However, to the best of our knowledge, two aspects of pipelines with combinatorial optimization layers have not been considered in the literature. 
First, the general learning by experience setting where only instances of the hard optimization problems are available has not been considered.
Second, there is no guarantee on the quality of the solution returned by the pipeline. The purpose of this paper is to address these two issues.

\paragraph{Contributions}
We make the following contributions.
\begin{enumerate}
    \item The design of the learning pipelines, and notably the choice of~\eqref{eq:easyProblem} and $f_{\bfw}$ is critical for the performance of the resulting algorithm. We illustrate on three applications among which our running example how to build such pipelines.
    \item A natural way of formulating the learning problem consists in minimizing the loss defined average cost of the solution $z_i$ returned by our pipeline for instance $x_i$
    $$\frac{1}{n}\sum_{i=1}^nc\Big(\argmin_{y \in \calY(x_i)}g\big(y,f_{\bfw}(x_i)\big),x_i\Big)$$
    We introduce a regularized version of this loss.
    And we show with extensive numerical experiments that, despite the non-convexity of this loss, when the dimension of $\calW$ is moderate, i.e., non-greater than $100$, solving this problem with a global black-box solver leads to surprisingly efficient pipelines.
    \item Leveraging tools from statistical learning theory, we prove the convergence of the learning algorithm toward the approximation with the best expected loss, and an upper bound on the convergence speed. % that does not depend on the diversity instances structures.
    \item We deduce from these statistical learning results that, under some hypotheses on the pipeline, the learned algorithm is an approximation algorithm for~\eqref{eq:hardProblem}. These hypotheses are notably satisfied by our solution pipeline for~\eqref{eq:twoStageSpanningTreeProblem}. 
\end{enumerate}

% \inlAP{We underline the fact that, in this paper, we do not try to approximate difficult constraint. We only try to approximate difficult objectives.}
Remark that, in this paper, we do not try to approximate difficult constraint. We only try to approximate difficult objectives.
The paper is organized as follows. Section~\ref{sec:designingAndExamples} introduce two additional examples and explain how to build pipelines.
Section~\ref{sec:learningProblem} formulates the learning by experience problem and introduces algorithms.
Section~\ref{sec:convergenceEstimator} introduces the convergence results and the approximation ratio guarantee. Finally, Section~\ref{sec:numericExperiments} details the numerical experiments.

\section{Designing pipelines with combinatorial optimization layers}
\label{sec:designingAndExamples}

In this section, we give a methodology to build pipelines with combinatorial optimization layers.
We illustrate it on our running example and on two applications previously introduced by the author.
% from the literature to illustrate the methodology.
% This section introduces two additional hard problems~\eqref{eq:hardProblem}, as well as the easy problems~\eqref{eq:easyProblem} used to approximate them.
% These were introduced in previous contributions~\citep{parmentierLearningApproximateIndustrial2019,parmentierScheduling2020}. 
We start with the description of these applications, which follows the papers which introduced them~\citep{parmentierLearningApproximateIndustrial2019,parmentierScheduling2020}.

\subsection{Stochastic vehicle scheduling problem.}
\label{sub:stoVSPproblem}

\paragraph{Stochastic vehicle scheduling problem}
Let $V$ be a set of tasks that should be operated using vehicles. 
For each task $v$ in $V$, we suppose to have a scheduled start time $\tb[v]$ in $\bbZ_+$ and a scheduled end time $\te[v]$ in $\bbZ_+$. We suppose $\te[v] > \tb[v]$ for each task $v$ in $V$. For each pair of tasks $(u,v)$, the travel time to reach task $v$ from task $u$ is denoted by~$\ttr[(u,v)]$.
Task $v$ can be operated after task $u$ using the same vehicle if
\begin{equation}\label{eq:stovspArc}
\tb[v] \geq \te[u] + \ttr[(u,v)].
\end{equation}
We introduce the digraph $D = (V,A)$ with vertex set $V = T \cup\{o,d\}$ where $o$ and $d$ are artificial origin and destination vertices.
The arc set $A$ contains the pair $(u,v)$ in $T^2$ if $v$ can be scheduled after task $u$, as well as the pairs $(o,v)$ and $(v,d)$ for all $v$ in $V$. 
An $o$-$d$ path $P$ represents a sequence of tasks operated by a vehicle.
A feasible solution is a partition of $V$ into $o$-$d$ paths.
If we denote by $c(P,\Gammah)$ the cost of operating the sequence corresponding to the $o$-$d$ path $P$, and by $\calPod$ the set of $o$-$d$ paths, the problem can be modeled as follows.
\begin{subequations}
    \label{eq:masterProblem}
    \begin{alignat}{2}
        \min_z\enskip & \sum_{P \in \calPod} c(P;\Gammah) z_P,
        & \quad &\\
        \mathrm{s.t.}\enskip & \sum_{P \ni v} z_P = 1, && \forall v \in V(\Gammah) \backslash\{o,d\}, \label{eq:coverConstraint} \\
        & z_P \in \{0,1\}, && \forall P \in \calPod(\Gammah), \label{eq:integer}
    \end{alignat}
\end{subequations}
% \comAP{Maybe put the remaining in appendix.}

Up to now, we have described a generic vehicle scheduling problem. Let us now define our stochastic vehicle scheduling problem by giving the definition of $c(P,\Gammah)$.
Let $\Omega$  be a set of scenarios.
For each task $v$, we have a random start time $\xib[v]$ and a random end time $\xie[v]$, and for each arc $(u,v)$, we have a random travel time $\xitr[(u,v)]$. 
Hence, $\xib[v](\omega)$, $\xie[v](\omega)$, and $\xitr[(u,v)](\omega)$ are respectively the beginning time of $v$, end time of $v$, and travel time between $u$ and $v$ under scenario $\omega$ in $\Omega$.
We define $\xie[o] = 0$ and $\xib[d] = +\infty$.

Given an $o$-$v$ path $P$, we define recursively the end-time $\tau_P$ of $P$ as follows.
\begin{equation}\label{eq:endTimeStoVsp}
	\tau_P = \left\{
	\begin{array}{ll}
		0, & \text{if $P$ is the empty path in $o$}, \\
		\xie[v] + \max (\tau_Q + \xitr[a] - \xib[v], 0), & \text{if $P = Q + a$ for some path $Q$ and arc $a$.} 
	\end{array}
	\right.
\end{equation}
Equation~\eqref{eq:endTimeStoVsp} models the fact that a task can be operated by a vehicle only when the vehicle has finished the previous task: The vehicle finishes $Q$ at $\tau_Q$, and arrives in $v$ at $\tau_Q + \xitr[a]$ with delay $\max (\tau_Q + \xitr[a] - \xib[v], 0)$.
The total delay $\Delta_P$ along a path $P$ is therefore defined recursively by
\begin{equation}\label{eq:totalDelayStoVsp}
	\Delta_P = \left\{
	\begin{array}{ll}
		0, & \text{if $P$ is the empty path in $o$}, \\
		\Delta_Q + \max (\tau_Q + \xitr[a] - \xib[v], 0), & \text{if $P = Q + a$ for some path $Q$ and arc $a$.} 
	\end{array}
	\right.
\end{equation}
Finally, we define the cost of an $o$-$d$ path $P$ as 
\begin{equation}\label{eq:costPathStoVsp}
	c(P;\Gammah) = \cveh + \cdel \bbE(\tau_P)
\end{equation}
where $\cveh$ in $\bbZ_+$ is the cost of a vehicle and $\cdel$ in $\bbZ_+$ is the cost of a unit delay.
Practically, we use a finite set of scenarios $\Omega$, and compute the expectation as the average on this set.

\paragraph{CO layer: usual vehicle scheduling problem}
The usual vehicle scheduling problem can also be formulated as~\eqref{eq:masterProblem}, the difference being that now the path can be decomposed as the sum of the arcs cost 
\begin{equation}
    \label{eq:pathSumArcCost}
    \overline{c}_P = \sum_{a \in P} \overline{c}_a \quad \text{with} \quad \overline{c}_a \in \bbR. 
\end{equation}
It can be reduced to a flow problem on $D$ and efficiently solved using flow algorithms or linear programming.
In Equation~\eqref{eq:pathSumArcCost} and in the rest of the paper, we use an overline to denote quantities corresponding to the easy problem.

\subsection{Single machine scheduling problem.}
\label{sub:schedulingProblem}

\paragraph{Scheduling problem $1|r_j|\sum_j C_j$.}
$n$ jobs must be processed in a single machine. Jobs cannot be interrupted once launched. 
Each job has a processing time $p_j$, and a release time $r_j$ in $\bbR$. 
That is, job $j$ cannot be started before $r_j$, and once started, it takes $p_j$ to complete it.
A solution is a schedule $s=(j_1,\ldots,j_n)$, i.e., a permutation of $[n]$ that gives the order in which jobs are processed. 
Using the convention $C_{j_0}=0$, the completion time of jobs in $s$ are defined as 
$$ C_{j_i} = \max(r_j,C_{j_{i-1}}) + p_{j_i}.$$
The objective is to find a solution minimizing $\sum_j C_j$.
This problem is strongly NP-hard.

\paragraph{Combinatorial optimization layer: $1||\sum_j C_j$.}
The easy problem is obtained when there is no release time, and only jobs processing times $\overline{p_j}$.
Jobs completion times are therefore given by
$$ \overline C_{j_i} = \overline C_{j_{i-1}} + \overline p_{j_i}.$$
Again, we use an overline to denote quantities of the easy problem.
An optimal schedule is obtained using the shortest processing time first (SPT) rule, that is, by sorting the jobs by increasing $p_{j}$.
% As an illustration, consider the 

% On many applications in the literature, $\calY(x)$ is typically embedded in $\bbR^{d(x)}$, and the objective is linear
% $$ g(y;\bftheta) = \langle \bftheta | y \rangle.$$

% Let us now make a few observations on some specificities of such pipelines.
% A solution algorithm for~\eqref{eq:hardProblem} should be able to address instances of very different sizes. 

% Instead of directly solving~\eqref{eq:hardProblem}, a pipeline such as the one  
% This will be handy because we are going to consider simultaneously several instances of the problem.
% When we use a machine learning pipeline such as the one on Figure~\ref{fig:pipeline} to solve~\eqref{eq:hardProblem}

% Model $f_{\bfw}$ belongs to a parametrized family of statistical model $\{f_{\bfw},\bfw \in \calW  \}$. One spec

\subsection{Constructing pipelines}
\label{sec:setting}
In this section, we explain how to build our learning pipelines.

\paragraph{Combinatorial Optimization layer and decoder.}
The choice of the combinatorial optimization layer and the decoder are rather applications dependent. 
Two practical aspects are important.
First, we must have a practically efficient algorithm to solve~\eqref{eq:easyProblem}.
Second, it must be easy to turn solutions of~\eqref{eq:easyProblem} into solution of~\eqref{eq:hardProblem}. 
That is, either the solutions of~\eqref{eq:easyProblem} and~\eqref{eq:hardProblem} coincide, or we must have a practically efficient algorithm $h$ that turns a solution of~\eqref{eq:easyProblem} into a solution of~\eqref{eq:hardProblem}.

\paragraph{Structure of $x$ and generalized linear model.}
As we indicated in the introduction, a practical difficulty in the definition of our statistical model $f_{\bfw}$ is that the size of its output $\bftheta$ in $\bfTheta(x)$ depends on the instance $x$.
Unfortunately, statistical models generally output vectors of fixed size.
Let us pinpoint a practical way of addressing this difficulty with a generalized linear model.
Let $\calI(x)$ be the \emph{structure} $x$, i.e., the set of dimensions $i$ of $\bfTheta(x)$.
We suggest defining a feature mapping 
$$ \bfphi : (i,x) \mapsto  \bfphi(i,x)$$
that associates to an instance and a dimension $i$ in $\calI(x)$ a feature vector $\bfphi(i,x)$ describing the main properties of $i$ as a dimension of $x$.
We then define
$$ f_{\bfw} : x \mapsto \bftheta \quad \text{with} \quad \bftheta = (\theta_i)_{i \in \calI(x)} \quad \text{and} \quad \langle \bfw | \bfphi(i,x)\rangle. $$
In summary, $f_{\bfw}$ can output parameters $\bftheta$ whose dimension depends on $x$ because it applies the same predictor $(i,x) \mapsto \langle \bfw | \bfphi(i,x) \rangle$ to predict the value for $\theta_i$ for the different dimensions in $\calI(x)$.

\paragraph{Illustration on our applications.}
For instance, let us consider our running example on \emph{two stage spanning tree problem}. 
Given an instance $x$, we must define the first and second stage costs $\bar c_e$ and $\bar d_e$ for each edge $e \in E$.
We can therefore define $\calI(x)$ as $\big\{(e,\texttt{stage})\colon e \in E,\, \texttt{stage} \in \{\texttt{first},\texttt{second}\}\big\}$.
The details of the features used is described in Table~\ref{tab:two_stage_features}.
For the stochastic vehicle scheduling problem, all we have to do is to define the arc costs $\bar c_a$. Hence, we can define $\calI(x) = A$.
And for the single machine scheduling problem, we only have to define the processing times $\bar p_j$.
Hence, $\calI(x) = \{1,\ldots,n\}$.

\begin{table}
    \begin{center}
        \begin{tabular}{p{0.36\textwidth}@{\qquad}>{\centering\arraybackslash}p{0.25\textwidth}@{\qquad}>{\centering\arraybackslash}p{0.28\textwidth}}
            \hline
            \multicolumn{1}{c}{Feature  description } & 
            \multicolumn{1}{c}{$\phi\big((e,\texttt{first}),x)$ } & 
            \multicolumn{1}{c}{$\phi\big((e,\texttt{second}),x)$ } \\
            \hline
            First stage cost & $c_e$ & 0 \\
            Second stage average cost & 0 & $\sum_s d_{es} / |S|$ \\
            Quantiles of second stage cost & 0 & $Q\big[(d_{es})_s\big]$ \\
            Quantiles of neighbors first stage cost & $Q\big[(c_{e'})_{e' \in \delta(u) \cup \delta(v)}\big]$ & 0 \\  
            Quantiles of neighbors second stage cost & 0&  $Q\big[(d_{e's})_{e' \in \delta(u) \cup \delta(v),s \in S}\big]$ \\
            ``Is edge in first stage MST ?''& $ \ind^{\mathrm{MST}}\big(e,(c_e)_{e \in E}\big)$ & 0\\
            Quantiles of ``Is edge in second stage MST quantile ?''& 0& $Q\Big[ \Big(\ind^{\mathrm{MST}}\big(e,(b_{es})_{e \in E}\big)\Big)_{s \in S}\Big]$ \\
            Quantiles of ``Is first stage edge in best stage MST quantile ?''& $Q\Big[ \Big(\substack{ \ind^{\mathrm{MST}}\big(e,(b_{es})_{e \in E}\big) \\ \text{and } c_e \leq d_{es}}\Big)_{s \in S}\Big]$ & 0 \\
            Quantiles of ``Is second stage edge in best stage MST quantile ?''&0 & $Q\Big[ \Big(\substack{ \ind^{\mathrm{MST}}\big(e,(b_{es})_{e \in E}\big) \\ \text{and } c_e > d_{es}}\Big)_{s \in S}\Big]$ \\
            \hline
            \multicolumn{3}{p{\textwidth}}{\scriptsize Note: MST stands for Minimum Weight Spanning Tree, $b_{es} = \min(c_e,d_{es})$, $\calQ[\bfa]$ gives the quantiles of a vector $\bfa$ seen as a sampled distribution, and $\ind^{\mathrm{MST}}(e,(\tilde c_e)_e)$ is equal to $1$ if $e$ is in the minimum spanning tree for edge weights $(\tilde c_e)_e$. }
            
        \end{tabular}
    \end{center}
    \caption{Two stage spanning tree features of edge $e = (u,v)$. \qquad}
    \label{tab:two_stage_features}
\end{table}

\paragraph{Encoding information on an element as part of an instance}
Let us finally introduce two generic techniques to build interesting features.
The features in Table~\ref{tab:two_stage_features} rely on these two techniques.
The first technique enables to compare dimension $i$ to the other ones in $\calI(x)$.
% A first technique enables to \emph{compare $i$ to $\bfrhoh_{e'}$ for $e' \in E_k$}. 
To that purpose, we define a statistic $\alpha : (i,\calX) \mapsto \alpha(i,x)$, and considers $f(\bfrhoh_e)$ as a realization of the random variable
$$\begin{array}{rcl}
    \calA : \calI(x) &\rightarrow& \bbR \\
    i &\mapsto& \alpha(i,x)
\end{array}$$
and take some relevant statistics on the realization $\calA(i)$ of $\calA$, such as the value of the cumulative distribution function of $F$ in $\alpha(i,x)$.
For instance, when considering a job $j$ of $1|r_j|\sum_j C_j$ with parameter $ (r_j,p_j)$, if we define $\alpha(j,x) = r_j + p_j$, we obtain as feature the rank (divided by $n$) of feature $j$ in the schedule where we sort the jobs by increasing $r_j + p_j$, a statistic known to be interesting and used in dispatching rules.

The second technique is to \emph{explore the role of $i$ in the solution of a very simple optimization problem}.  A natural way of building features is to run a fast heuristic on the instance $x$ and seek properties of $i$ in the resulting solution.
For instance, the preemptive version of $1|r_j|\sum_j C_j$, where jobs can be stopped, is easy to solve. Statistics such as the number of times job $j$ is preempted in the optimal solution can be used as features.

\paragraph{Equivariant layers.}
A lesson from geometric deep learning~\citep{bronstein2021geometric} is that good neural network architecture should respect the symmetries of the problem.
Let $\calS$ be a symmetry of the problem. 
A layer $h$ in a neural network is said to be equivariant with respect to $\calS$ if $h(\calS(x)) = \calS(h(x))$. 
In our combinatorial optimization setting, there is one natural symmetry. The solution predicted $y$ should not depend on the indexing of the variables using in the combinatorial optimization problem : Given a permutation of these variables in the instance $x$, the solution $y$ should be the permuted solution.
Combinatorial optimization layers are naturally equivariant with respect to this symmetry. The generalized linear model above is a simple example of equivariant layer.

\section{Learning by experience}
\label{sec:learningProblem}

% Having defined these structured approximations, we now focus on the problem of learning such approximations.
We now focus on how to learn pipelines with a combinatorial optimization layer.
Given a training set composed of representative instances, the learning problem aims at finding a parameter $\bfw$ such that the output $z(\bfw)$ of our pipeline has a small cost. % $c(y(\bfw);x)$.

As we mentioned in the introduction, the literature focuses on the learning by imitation setting. In that case, the training set $(x_1,y_1),\ldots,(x_n,y_n)$ contains instances and target solution of the prediction problem~\eqref{eq:easyProblem}, the learning problem can be 
%seen as a structured learning problem~\citep{nowozinStructuredLearningPrediction2010}.
% The learning problem is generally 
formulated as 
$$ \min \frac{1}{n} \sum_{i=1}^n \ell(\bfvarrhoe_i,y_i) \quad \text{where} \quad \bfvarrhoe_i = \tilde \varphi_w(x_i), $$
and $\ell(\bfvarrhoe,y_i)$ is a loss function. 
Losses that are convex in $\bftheta$ and lead to practically efficient algorithms have been proposed when~\eqref{eq:easyProblem} is linear on $\bfvarrhoe$, which is the case on most applications. 
% The SPO+ loss proves successful when the training set contains target $\bftheta_i$ instead of target $y_i$ \citep{elmachtoubSmartPredictThen2021}.
% i.e., $\fe(x;\calE,\bfvarrhoe) = \langle  \bfvarrhoe | \Upsilon(x) \rangle$ for some embedding $\Upsilon$, pratca
% Losses  $\bftheta$.
% .
Typical examples include the structured Hinge loss~\citep{nowozinStructuredLearningPrediction2010} or the Fenchel-Young losses~\citep{berthetLearningDifferentiablePerturbed2020}.
The SPO+ loss proves successful when the training set contains target $\bftheta_i$ instead of target $y_i$ \citep{elmachtoubSmartPredictThen2021}.

In this paper, we focus on the learning by experience setting, where the training set $(x_1,\ldots,x_n)$ contains instances but not their solutions.
% Since the post-processing $h$ is deterministic, and for notational convenience,
% $$ c(y;x) = c(h(y);x) $$

% A natural way of formulating the learning problem in that context is to use Fenchel-Young losses.
% $$ \min \frac{1}{n} \sum_{i=1}^n F(\bfvarrhoe_i) + \Omega(x_i) - \langle\bfvarrhoe_i | \Upsilon(x_i)\rangle  $$
% where 
% This approach has the advantage using only the easy problem solver. Alternative such as the structured SVM require to solve many instances a variable of the easy problem

\subsection{Learning problem and regularized learning problem}

Let $\Gammah_1,\ldots,\Gammah_n$ be our \emph{training set} composed of $n$ instances of~\eqref{eq:hardProblem}. 
Without loss of generality, we suppose that $c(y;\Gammah) \geq 0$ for all instances $\Gammah$ and feasible solution $y \in \calY(x)$. 
We also suppose to have a mapping $u : x \mapsto u(x) \geq 0$ that is a coarse estimation of the absolute value of an optimal solution of $x$. 
We define the loss function as the weighted cost of the easy problem solution as a solution of the hard problem.
\begin{equation}\label{eq:lossFunctionWithoutPostProcessing}
    \ell(\bfw,\Gammah) := \frac{1}{u(x)}\max\Big\{\fh\big(y;\Gammah\big)\colon y \in \argmin_{\tilde y \in \calY(x)} \fe\big(\tilde y,f_{\bfw}(\Gammah)\big)\Big\}.
\end{equation}
The \emph{learning problem} consists in minimizing the expected loss on the training set
\begin{equation}\label{eq:learningProblem}
    \min_{\bfw \in \bfW} \frac{1}{n} \sum_{i= 1}^n \ell(\bfw,\Gammah_i).
\end{equation}
The instances in the training set may be of different size, leading to solutions costs which different order of magnitudes. 
The weight $\frac{1}{u(x)}$ enables to avoid giving too much importance to large instances.

When the approximation is flexible and the training set is small, the solution of~\eqref{eq:learningProblem} may overfit the training set, and lead to poor performance on instances that are not in the training set.
In that case, the usual technique to avoid overfitting is to regularize the problem. 
One way to achieve this is to make the prediction ``robust'' with respect to small perturbations: 
We want the solution returned to be good even if we use $\bfw + \bfZ$ instead of $\bfw$, where $\bfZ$ is a small perturbation.
Practically, we assume that $\bfZ$ is a standard Gaussian, $\sigma>0$ is a real number, and we define the \emph{perturbed loss}
\begin{equation}\label{eq:perturbedLossFunction}
    \ell^{\mathrm{pert}}(\bfw,\Gammah) = \bbE_{\bfZ}\Big[\frac{1}{u(x)}\max\Big\{\fh\big(y;\Gammah\big)\colon y \in \argmin_{\tilde y \in \calY(x)} \fe\big(\tilde y,f_{\bfw + \sigma \bfZ}(\Gammah)\big)\Big\}\Big].
    % \bbE_{\bfZ}\ell(\bfw + \sigma\bfZ,\Gammah).
\end{equation}
This perturbation can be understood as a regularization of the easy problem \citep{berthetLearningDifferentiablePerturbed2020}.
The \emph{regularized learning problem} is then formulated as follows. 
\begin{equation}\label{eq:regularizedLearningProblem}
    \min_{\bfw \in \bfW} \frac{1}{n} \sum_{i=1}^n \ell^{\mathrm{pert}}(\bfw,\Gammah)
\end{equation}

\subsection{Algorithms to solve the learning problem}
\label{sub:learningAlgorithm}

% The following result gives insights on the kind of algorithm that should be used to solve the learning problem~\eqref{eq:learningProblem}.

\begin{prop}\label{prop:piecewiseConstant}
    If
    % \begin{itemize}
        % \item 
        $\bfw \mapsto f_{\bfw}(x)$ and $\bftheta \mapsto \fe(y,\bftheta)$ are piecewise linear for all $y$ in $\calY(x)$, then the objective of~\eqref{eq:learningProblem} is piecewise constant in $\Gammah$. 
\end{prop}
\begin{proof}
    Since the composition of two piecewise linear functions is piecewise linear, $\bfw \mapsto \fe(x;f_{\bfw}(\Gammah))$ is piecewise linear.
    Hence, there exists a partition of the space into a finite number of polyhedra such that the set 
    $\Big\{\fh\big(x;\Gammah\big)\colon x \in \argmin_{x \in \calX(\calE)} \fe\big(x;\varphi_{\bfw}(\Gammah)\big)\Big\}$ is constant on each polyhedron.
    The definition of $\ell(\bfw,\Gammah)$ then ensures that $\bfw \mapsto \ell(\bfw,\Gammah)$ is piecewise constant on the interior of each polyhedron of the partition, and lower semi-continuous, which gives the result.
\end{proof}

Proposition~\ref{prop:piecewiseConstant} is bad news from an optimization point of view.
We need a black-box optimization algorithm that uses a moderate amount of function evaluations, does not rely on ``slope'' (due to null gradient), and takes a global approach (due to non-convexity).
% We need an optimization algorithm that does not use (approximations of) gradients, but only (a moderate amount of) evaluations of the function.
We therefore suggest using either a heuristic that searches the state space such as the DIRECT algorithm~\citep{jonesLipschitzianOptimizationLipschitz1993a}, or a Bayesian optimization algorithm that builds a global approximation of the objective function and uses it to sample the areas in the space of $\bfw$ that are promising according to the approximation.
The numerical experiments evaluate the performance of these two kinds of algorithms.

Let us now consider the regularized learning problem~\eqref{eq:regularizedLearningProblem}.
Since the convolution product of two functions is as smooth as the most smooth of the two functions, $\bfw \mapsto \ell^{\mathrm{pert}}(\bfw,\Gammah)$ is $C^{\infty}$. 
It can therefore be minimized using a stochastic gradient descent~\citep{dalleLearningCombinatorialOptimization2022}.
On our applications, and using a generalized linear model, we obtained better results by solving a sample average approximation of this perturbed learning problem using the heuristics mentioned above.
This is not so surprising because in that case, the objective of the learning problem is composed of several plateaus with smooth transition inbetween, which is not much easier to solve in practice.
Remark that stochastic gradient descent is the method of choice when using a large neural network.
% We obtained much better results by solving a sample average approximation of this perturbed learning problem using the heuristics mentioned above.

\subsection{Practical remarks for a generic implementation}
\label{sub:practicalRemarks}

\paragraph{Perturbation strength.} Section~\ref{sec:convergenceEstimator} provides a closed formula to set the perturbation strength $\sigma$.
% We observed in the numerical experiments that using no perturbation also leads to good results in practice on our problems.

\paragraph{Skipping the bilevel optimization}
Using a bilevel optimization enables to define $\ell(\bfw,\Gammah)$ unambiguously even when the easy problem~\eqref{eq:easyProblem} admits several optimal solutions.
Since the bilevel optimization is not easy to handle, we use in practice the loss
$$\tilde\ell(\bfw,\Gammah) = \frac{1}{u(x)}\fh\big(\calA\circ\varphi_{\bfw}(\Gammah)\big)$$
that takes the solution returned by the algorithm $\calA$ we use for~\eqref{eq:easyProblem}. 
% That is we
% which is a black-box and not a function since different runs of algorithm $\calA$ may return different solutions 
Its value may therefore depend on $\calA$.
% when~\eqref{eq:easyProblem} admits several optimal solutions.

\paragraph{Post-processing}
On many applications, the post-processing $h$ is time-consuming, and there exists an alternative post-processing $\tilde h$ that is much faster, even if the resulting solution $z$ may have a larger cost.
A typical example is our $1|r_j|\sum_j C_j$ application, where $\calY(x) = \calZ(x)$, and the post-processing is only a local descent. The post-processing is therefore not mandatory, and we could use $\tilde h = \text{Identity}$.
In that context, using $\tilde h$ instead of $h$ during the learning phase leads to a much faster learning algorithm, while not necessarily hurting the quality of the $\bfw$ learned. 

% As an alternative, we can use in the learning problem the solution returned by the solution pipeline after the post-processing $\psi$ instead of the solution of the easy problem. 
% The loss becomes
% \begin{equation}\label{eq:lossFunctionWithPostProcessing}
%     \ell^{\psi}(\bfw,\Gammah) := \frac{1}{u(x)}\fh\big(y_{\bfw}(x),x\big).
% \end{equation}
% where $y_{\bfw}(x)$ is the solution returned by our algorithm for~\eqref{eq:easyProblem} when called with $\bftheta = f_{\bfw}(x)$.
% % The loss $\ell^{\psi}$ is again a black-box and not a function.
% % In our numerical experiments, 
% The resulting learning algorithm may be much more time-consuming depending on $\psi$. 
% %is the most time-consuming part of the solution pipeline.
% % Given our limited learning time budget, this constrains us to use smaller learning set, which is not beneficial.

\paragraph{Sampling in the prediction pipeline.}
If we are ready to increase the execution time, the perturbation of $\bfw$ by $\bfZ$ can also be used to increase the quality of the solution returned by our solution pipeline.
We can draw several samples $\bfZ_i$ of $\bfZ$, apply the solution pipeline with $\bfw + \sigma\bfZ_i$ instead of $\bfw$, and return the best solution found across the samples at the end.
We provide numerical results with this perturbed algorithm on the $1|r_j|\sum_j C_j$ problem in Section~\ref{sec:numericExperiments}.

\section{Learning rate and approximation ratio}
\label{sec:convergenceEstimator}

This section introduces theoretical guarantees on the average optimality gap of the solution returned by the learned algorithm when $\bfw$ is chosen as in Section~\ref{sec:learningProblem}.
% For statistical
% The notion of structured approximation lets too much flexibility in the choice of $\varphi_{\bfw}$ and $\fe$ to prevent overfitting when using the non-regularized problem~\eqref{eq:learningProblem}.
Two conditions seem necessary to obtain such guarantees. 
First, it must be possible to approximate the hard problem by the easy one.
That is, there must exist a $\tilde\bfw$ such that an optimal solution of $\varphi_{\tilde\bfw}(\Gammah)$ provides a good solution of $\Gammah$.
And second, when such a $\tilde\bfw$ exists, our learning problem must be able to find it or another $\bfw'$ that leads to a good approximation.
Our proof strategy is therefore in two steps.
First, we show that the solution of our learning problem converges toward the ``best'' $\bfw$ when the number of instances in the solution set increases.
And then we show that if there exists a $\tilde\bfw$ such that the expected optimality gap of the solution returned by our solution approach is bounded, then the expected optimality gap for the learned $\bfw$ is also bounded.
For statistical reasons discussed at the end of Section~\ref{sub:approx_ratio},
we carry this analysis using the regularized learning problem~\eqref{eq:regularizedLearningProblem}. % because $\ell$ is not smooth (Proposition~\ref{prop:piecewiseConstant}).
%  dicussed the need for and the influence of the pertubation.

\subsection{Background on learning with perturbed bounded losses}
\label{sub:backgroundStatisticalLearning}

Let $\xi$ be a random variable on a space $\Xi$ and $\bfW$ a non-empty compact subset of $\bbR^d$,  $\bfW \subseteq \bbB_{\infty}(M)$ where $\bbB_{\infty}(M)$ is the $\|\cdot\|_{\infty}$ ball of radius $M$ on $\bbR$. 
Let $\ell : \Xi \times \bbR^d \rightarrow [0,1]$ be a loss function, we define the perturbed loss as 
\begin{equation}
    \label{eq:perturbedLoss}
    \ell^{\mathrm{pert}}(\bar\xi,\bfw) = \bbE\bigl[\ell(\bar\xi,\bfw + \sigma Z)\bigr] \quad \text{with} \quad \sigma > 0.
\end{equation}
We suppose that $\ell(\cdot,\bfw)$ is integrable for all $\bfw \in \bfW$.
We define the \emph{expected risk}  $L(\bfw)$ and the \emph{expected risk minimizer}  $\bfw^*$ as
\begin{equation}\label{eq:expectedRiskMinimization}
    \bfw^* \in \argmin_{\bfw \in \bfW} L(\bfw) \quad \text{with} \quad L(\bfw) = \bbE \bigl[\ell^{\mathrm{pert}}(\xi,\bfw)\bigr].
\end{equation}
Let $\xi_1,\ldots,\xi_n$ be $n$ i.i.d.~samples of $\xi$. We define the \emph{empirical risk} $\hat L_n(\bfw)$ and the empirical risk minimizer $\hat \bfw_n$ as
\begin{equation}\label{eq:empriricalRiskMinimization}
    \hat \bfw_n  \in \argmin_{\bfw \in \bfW} \hat L_n(\bfw) \quad \text{with} \quad \hat L_n(\bfw) = \frac{1}{n}\sum_{i=1}^n\ell^{\mathrm{pert}}(\xi_i,\bfw).
\end{equation}
Note that both $L_n(\bfw)$ and $\hat \bfw_n$ are random due to the sampling of the training set $\xi_1,\ldots,\xi_n$.
The following result bounds the excess risk incurred when we use $L_n(\bfw)$ instead of $L(\bfw)$.
\begin{theo}\label{theo:perturbedBoundedLossLearningRate}
    Suppose that $\bfW \subseteq \bbB_{\infty}(M)$ where $\bbB_{\infty}(M)$ is the $\|\cdot\|_{\infty}$ ball of radius $M$ on $\bbR$. Given $0< \bfdelta <1$, with probability at least $1-\delta$, we have the following bound on the excess risk.
    \begin{equation}\label{eq:upperBoundTheoPerturbedBoundedLossLearningRate}
        L(\hat \bfw_n) - L(\bfw^*) \leq C\frac{Md}{\sigma\sqrt{n}}  + \sqrt{\frac{2\log(2/\delta)}{n}}
    \end{equation}
    with $C =48\int_0^1\sqrt{-\log{x}}dx$.
\end{theo}
We believe that Theorem~\ref{theo:perturbedBoundedLossLearningRate} is in the statistical learning folklore, but since we did not find a proof, we provide one based on classical statistical learning results in Appendix~\ref{sec:ProofOfTheoremtheo:perturbedBoundedLossLearningRate}.

\paragraph{Learning rate of our structured approximation}
In order to apply Theorem~\ref{theo:perturbedBoundedLossLearningRate} to the learning problem of Section~\ref{sec:learningProblem}, we must endow the set $\alephh$ of instances with a distribution.
% We denote by $\frakE$ the set of possible structures. 
% We assume it finite.
% For each structure $\calE$, we denote by
% $$\frakRh(\calE) = \Big\{ \bfvarrhoh = (\bfrhoh_e)_{e \in E_k,k \in K} \colon \bfrhoh_e \in \calR_k,\forall e \in E_k, \forall k \in K \Big\} \subseteq \bbR^{\sum_{k \in K} \degh_k \times E_k} $$ 
% the set of possible parametrizations of $\calE$.
% Finally, 
% we recall that
% $\alephh = \big\{(\calE,\bfvarrhoh)\colon \calE \in \frakE, \bfvarrho \in \frakRh(\calE)\big\}$
% is the set of possible instances.
Recall that the loss $\ell$ and the perturbed loss $\ell^{\mathrm{pert}}$ have been defined in Equations~\eqref{eq:lossFunctionWithoutPostProcessing} and~\eqref{eq:perturbedLossFunction}.
We assume that the instance $\Gammah$ is a random variable with probability distribution $\mu$ on $\alephh$, and that both $\ell(\cdot,\bfw)$ and $\ell^{\mathrm{pert}}(\cdot,\bfw)$ are integrable for all $\bfw$.
With these definitions, using $\alephh$ as $\Xi$, instances $\Gammah$ as random variables $\xi$, and if we suppose that the training set $\Gammah_1,\ldots,\Gammah_n$ is composed of $n$ i.i.d.~samples of $\Gammah$, we have recast our regularized learning problem~\eqref{eq:regularizedLearningProblem} as a special case of~\eqref{eq:empriricalRiskMinimization}.
We can therefore apply Theorem~\ref{theo:perturbedBoundedLossLearningRate} and deduce that the upper bound~\eqref{eq:upperBoundTheoPerturbedBoundedLossLearningRate} on the excess risk applies.

% \inlAP{ Reformulate this paragraph.}
This result underlines a strength of the architectures of Section~\ref{sec:setting}. 
Because they enable to use approximations parametrized by $\bfw$ whose dimension is small and does not depend on $x$, the bound on the excess risk only depends on the dimension of $\bfw$ and not on the size of the instances used.
Hence, \emph{a pipeline with these architectures enables to make predictions that generalize (in expectation) on a test set whose instances structures $\calI(x)$ are not necessarily present in the training set}.
This is confirmed experimentally in Section~\ref{sec:numericExperiments}, where the structures of the instances in the test set of the stochastic VSP (the graph $D$) do not appear in the training set.

\begin{rem}
    Theorem~\ref{theo:perturbedBoundedLossLearningRate} does not take into account the fact that, practically and in our numerical experiments, we use a sample average approximation on $\bfZ$ of the perturbed loss instead of the true perturbed loss. 
\end{rem}

\subsection{Approximation ratio of our structured approximation}
\label{sub:approx_ratio}
Let $d(x) = |I(x)|$ and $c^* (x) = \argmin_{y \in \calY(x)}c(y,x)$ be the cost of an optimal solution of~\eqref{eq:hardProblem}.
\begin{theo}\label{theo:approximationRatio}
    Suppose that for all $x$ in $\calX$ (outside a negligible set for the measure on $ \calX$), % and $i \in \calI(x)$, 
    \begin{enumerate}
        \item $f_{\bfw}(x) = \big(\langle \bfw | \bfphi(i,x)\rangle\big)_{i \in \calI(x)} $ 
        % and $ x \mapsto f_{\bfw}(e,\Gammah)$ 
        and $\|\bfphi(i,x)\|_2 \leq \kappa_{\phi}$,
        \item and there exists $\tilde \bfw$, $a>0$, $b>0$, and $\beta \in \{1,2\}$ such that, for any $\bfp \in \bbR^{d(x)}$,
        $$ c(y, x) - c^*(x)\leq a u(x) + b \|\bfp\|_\beta \quad \text{for any }y \in \argmin_{\tilde y \in \calY(x)}g\big(\tilde y,f_{\tilde \bfw}(x) + \bfp\big) $$
    \end{enumerate}
    Then, under the hypotheses of~Theorem~\ref{theo:perturbedBoundedLossLearningRate}, with probability at least $1-\delta$ (on the sampling of the training set)
    $$ \underbrace{L(\hat\bfw_n) - \bbE\biggl[\frac{c^*(x) }{u(x)}  \biggr]}_{\text{Perturbed prediction optimality gap}} 
    \leq 
    \underbrace{C\frac{Md}{\sigma\sqrt{n}} + \sqrt{\frac{2\log(2/\delta)}{n}}}_{\text{Training set error}} 
    + \underbrace{a}_{\substack{Appro-\\ximation\\error}} 
    + \underbrace{b\sigma\kappa_{\phi}\sqrt{d}\bbE\Biggl[\frac{[d(x)]^{1/\beta}}{u(x)}\Biggr]}_{\text{Perturbation error}}$$
\end{theo}
Before proving the theorem, let us make some comments.
First, we explain why the hypotheses are meaningful.
The first hypothesis only assumes that the model is linear, and that, with probability $1$ on the choice of $x$ in $\calX$ the features are bounded. 
Such a hypothesis is reasonable as soon as we restrict ourselves to instances whose parameters are bounded.
Let us recall that $u(x)$ is a coarse upper bound on $c^*(x)$.
When $u(x) = c^*(x)$ and $\bfp = 0$, the second hypothesis only means that our non-perturbed pipeline with parameter $\tilde \bfw$ is an approximation algorithm with ratio $1+a$.
With a $\bfp \neq 0$, the second hypothesis is stronger: It also ensures that this approximation algorithm guarantee does not deteriorate too fast.
Later in this section, we prove that this hypothesis is satisfied for our running example.

\paragraph{Approximation ratio guarantee}
Using $u(x) = c^*(x)$ makes clear the fact that \emph{Theorem~\ref{theo:approximationRatio} provides an approximation ratio guarantee in expectation.}
Furthermore, 
% \paragraph{Choice of the perturbation strength.}
it gives \emph{a natural way of setting the strength }$\sigma$ of the perturbation: The bound is minimized when we use
\begin{equation}\label{eq:optimalSigmaN}
    \sigma_n = \sqrt{
    \frac
        {CM\sqrt{d}}
        {\sqrt{n}b\kappa_{\phi}\bbE\Bigl[\frac{[d(x)]^{1/\beta}}{u(x)}\Bigr]}
}.
\end{equation}
% \paragraph{Large training set regime.}
Using this optimal perturbation and $\delta = \frac{1}{n}$, the upper bound on $L(\hat\bfw_n) - \bbE\bigl[\frac{c^*(x) }{u(x)}  \bigr]$ is in $$a + O\Big(n^{-1/4}\big(1 + \log(n)\big)\Big)\xrightarrow[n\to +\infty]{}  a.$$ 
% \begin{coro}
%     Under the hypotheses of Theorem~\ref{theo:approximationRatio}, and using the $\sigma_n$ of Equation~\eqref{eq:optimalSigmaN}, 
%      $$ L(\hat\bfw_n) - \bbE\biggl[\frac{c^*(x) }{u(x)}  \biggr] \xrightarrow[n\to +\infty]{}  a. $$
% \end{coro}
In other words, \emph{in the large training set regime the learned $\hat \bfw_n$ recovers the approximation ratio guarantee $a$ of $\tilde \bfw$}.

\paragraph{Large instances}
Theorem~\ref{theo:approximationRatio} always provides guarantees when $d(x)$ is bounded on $\calX$. 
However, it may fail to give guarantees when $d(x)$ is unbounded on $\calX$ since $\bbE\bigl[\frac{[d(x)]^{1/\beta}}{u(x)}\bigr]$  may not be finite.
% We now explain on which kind of combinatorial optimization problems $\bbE\bigl[\frac{[d(x)]^{1/\beta}}{u(x)}\bigr]$ is likely to remain small even when $\calX$ contains large instances with $d(x) \rightarrow \infty$.
Since $u(x)$ is a coarse upper bound on $c^*(x)$, the term
$\bbE\bigl[\frac{[d(x)]^{1/\beta}}{u(x)}\bigr]$ remains finite when $d(x)$ is unbounded only if the cost of an optimal solution $c^*(x)$ grows at least as fast as the number of parameters of the instance $d(x)$ to the power ${1/\beta}$. 
From that point of view, the single machine scheduling problem $1|r_j|\sum_{C_j}$ is ideal. Indeed, in that case $d(x)$ is the number of job, and $c^*(x)\sim[d(x)]^2$, hence $\bigl[\frac{[d(x)]^{1/\beta}}{u(x)}\bigr] $ becomes smaller and smaller when the size of $x$ increases.
On the two stage spanning tree problem, the situation is slightly less favorable. Indeed, $d(x)$ is equal to twice the number of edges. Since an optimal spanning tree contains $|V|-1$ edges, we expect $c^*(x)$ to be of the order of magnitude of $d(x)$ on sparse graphs (graphs such that $|E| \sim |V|$, like grids for instance), and $\sqrt{d(x)}$ on dense graphs (graph such that $|E| \sim |V|^2$, like complete graphs). 
Later in this section, we prove that the second hypothesis is satisfied for the two stage spanning tree problem with $\beta = 1$.
Hence, Theorem~\ref{theo:approximationRatio} gives an approximation ratio guarantee for sparse graphs.
The situation is roughly the same for the stochastic vehicle scheduling problem.
A typical example where $\bbE\frac{[d(x)]^{1/\beta}}{u(x)}$ may not be finite is the shortest path problem on a dense graph. On many applications such as finding an optimal journey on a public transport system, the number of arcs tends to remain bounded, say $\leq 10$, while the number of arcs in the graph $d(x)$ grows with the size of the instance. 

\paragraph{Optimal resolution of the learning problem}
Theorem~\ref{theo:approximationRatio} applies for the optimal solution $\hat\bfw_n$ of the learning problem.
We let it to future work to design an exact algorithm that guarantees that the $\bfw$ returned is within an optimality gap $\gamma$ with the optimal solution.
We would then obtain a variant of Theorem~\ref{theo:approximationRatio} proving the approximation ratio result for the $\bfw$ returned, with an additional term in $\gamma$ in the upper bound taking into account the optimality gap.

% Note that the hypothesis that $\varphi_{\bfw}(e,\Gammah)$ is Lipschitz is satisfied when $\overline{\varphi_{\bfw}(e,\Gammah)}$ is the dot product $\langle \bfw | \bfphi(e,\Gammah) \rangle$ and the feature map is bounded $|\bfphi(e,\Gammah)| \leq \kappa^{\varphi}$. 
% Let us recall that $u(\calE)$ is a coarse approximation of the cost of an instance with structure $\calE$. 
% On many operations research problems, we there expect to scale linearly with $|\calE|$.
% In this case, the perturbation error is therefore in $O\Big(\bbE[1/\sqrt{|\calE|}]\big)$ and its impact on the gap decreases with the size of the instance.
% Remark that $\sqrt{|\calE|}$ and $u(\calE)$ both measure the size of the instance. 

\paragraph{Influence of the perturbation}
% \label{sub:perturbationInfluence}
% We mention at the beginning of the section that we need the perturbation to make the loss smoother.
Since we have made very few assumptions on $\fh$, $\fe$, and $f_{\bfw}$, we do not have control on the size of the family of functions $\{\ell_{\bfw}\colon\bfw\}$,
This family may be very large, and therefore able to fit any noise, which would lead to slower learning rate.
Without additional assumptions, we therefore need to regularize the family.
In particular, we need to smooth the piecewise constant loss (Proposition~\ref{prop:piecewiseConstant}).
As we have seen in this section, perturbing $\bfw$ does the job, but comes at a double cost in Theorem~\ref{theo:approximationRatio}: a perturbation error, and larger than hoped training set error in $O(d/\sqrt{n})$.
The term in $O(d/\sqrt{n})$ is slightly disappointing because the proof techniques used in statistical learning theory typically lead to bounds in $O(\sqrt{d}/\sqrt{n})$.
This is for instance the case for the metric entropy method used to prove Theorem~\ref{theo:perturbedBoundedLossLearningRate} when the gradient of the loss is Lipschitz in $\bfw$.
The Gaussian perturbation restores the Lipschitz property for the perturbed loss, but it comes at the price of an additional $\sqrt{d}$ in the bound derived by the metric entropy method, as can be seen in the proof of Lemma~\ref{lem:boundedPerturbedLipchitz} in Appendix~\ref{sec:ProofOfTheoremtheo:perturbedBoundedLossLearningRate}.
Designing a learning approach that avoids the additional $\sqrt{d}$ term is an interesting open question.
An alternative would be to make more assumptions on $\fh$, $\fe$, and $\varphi$ with the objective of making the perturbation optional in the proof.

\paragraph{Proof of Theorem~\ref{theo:approximationRatio}}
    We have
    $$ L(\hat\bfw_n) - \bbE\biggl[\frac{c^*(x) }{u(x)}  \biggr] =  \underbrace{L(\hat\bfw_n)- L(\bfw^*) }_{\leq C\frac{Md}{\sigma\sqrt{n}} + \sqrt{\frac{2\log(2/\delta)}{n}}} + \underbrace{L(\bfw^*) - L( \tilde \bfw)}_{\leq 0} + L(\tilde\bfw) - \bbE\big[\frac{c^*(x) }{u(x)}  \big] $$
    %+ \underbrace{L(\bfw^*) - \bbE\biggl[\frac{c^*(x) }{u(x)}  \biggr]}_{\leq L(\tilde\bfw) - \bbE\big[\frac{c^*(x) }{u(x)}  \big]}   $$
    We therefore need to upper bound $L(\tilde\bfw) - \bbE\bigl[\frac{c^*(x) }{u(x)}  \bigr]$ by $ a+b\sigma\kappa_{\phi}\sqrt{d}\bbE\bigl[\frac{[d(x)]^{1/\beta}}{u(x)}\bigr]$. We have % = \bbE\bigl[\frac{ - c^*(x) }{u(x)}  \bigr]$.
    \begin{align*}
       \big\| f_{\tilde \bfw + \sigma Z}(x) - f_{\tilde \bfw}(x) \big\|_{\beta} 
       &= \Big\| \sigma \big(\langle \bfZ | \bfphi(i,x)\big)_{i \in \calI(x)}  \Big\|_{\beta} \\
       &\leq  \Big\| \sigma \big(\|Z\|_2\|\phi(i,x)\|_2\big)_{i \in \calI(x)}  \Big\|_{\beta} \\
       &\leq \Big\| \sigma \big(\|Z\|_2\kappa_{\phi}\big)_{i \in \calI(x)}  \Big\|_{\beta}  = \sigma \kappa_{\phi}\|Z\|_2[d(x)]^{1/\beta} 
    \end{align*}
    Let $y$ be in $\argmin_{y \in \calY(x)}g\big(y,f_{\tilde \bfw + \sigma \bfZ}(x)\big)$. The second hypothesis of the theorem and the previous inequality give
    \begin{align*}
        c(y, x) - c^*(x) \leq a u(x) + b \kappa_{\phi}\|Z\|_2[d(x)]^{1/\beta}.
    \end{align*}
    Since $\bfZ$ is a standard Gaussian, we get $\bbE \|\bfZ\| \leq \sqrt{d}$ (Equation~\eqref{eq:upperBoundExpectationNormGaussian} in Appendix~\ref{sec:ProofOfTheoremtheo:perturbedBoundedLossLearningRate}), and the result follows by dividing the previous equality by $u(x)$ and taking the expectation. \hfill \qed

\subsection{Existence of a $\tilde\bfw$ with an approximation ratio guarantee}

In this section, we prove that the hypotheses of Theorem~\ref{theo:approximationRatio} are satisfied for the maximum weight two stage spanning tree problem. 
We then give a criterion which ensures that these hypotheses are satisfied.

\paragraph{Maximum weight two stage spanning tree}
In this section, we restrict ourselves to maximum weight spanning tree instances, that is, instances of the minimum weight spanning tree with $c_{e} \leq 0$ and $d_{es} \leq 0$ for all $e$ in $E$ and $s$ in $S$.
Given an instance, let $I(x) = \big((e,\texttt{stage})\colon e \in E, \, \texttt{stage} \in \{1,2\} \big)$, leading to $\bar c_e = \big\langle \bfw | \bfphi\big((e,1),x\big)\big\rangle$ and $\bar d_e = \big\langle \bfw | \bfphi\big((e,2),x\big)\big\rangle$.
We define a feature
$$ \phi(e,1) = c_e \quad \text{and} \quad \phi(e,2) = \frac{1}{|S|}\sum_{s \in S}d_{es},$$
and define $\tilde \bfw$ to be equal to $1$ for this feature and $0$ otherwise.
The following proposition shows that the second hypothesis of Theorem~\ref{theo:approximationRatio} is then satisfied with $a =1/2$ and $b=1$.

\begin{prop}\label{prop:twoStageSpanningTree_ratio}
    For any instance $x$ of the maximum weight spanning tree problem, we have
    $$ c(y, x) - c^*(x)\leq \frac{1}{2} |c^*(x)| + \|\bfp\|_1 \quad \text{for any }y \in \argmin_{y' \in \calY(x)}g\big(y,f_{\tilde \bfw}(x) + \bfp\big). $$
\end{prop}
Our proof shows that the results stands for $a= \frac{|S|-1}{2|S|-1}$ when $|S|$ is upper-bounded by $M$. 
Combined with Theorem~\ref{theo:approximationRatio}, Proposition~\ref{prop:twoStageSpanningTree_ratio} ensures that, when the training set is large, the learned $\hat w_n$ has the approximation ratio guarantee proved by \citet{escoffierTwostageStochasticMatching2010} for the two stage maximum weight spanning tree.
The proof of Proposition~\ref{prop:twoStageSpanningTree_ratio} is an extension of the proof of \citet[Theorem 6]{escoffierTwostageStochasticMatching2010} to deal with non-zero perturbations $\bfp$. 

\begin{proof}[Proof of Proposition~\ref{prop:twoStageSpanningTree_ratio}]
    The proof will use the following well known result
    \begin{lem}\label{eq:optimumOfTwoFunctionsWithBoundedDifference}
        Let $x\mapsto f_1(y)$ and $y\mapsto f_2(y)$ be functions from compact set $K$ to $\bbR$, and let $y_1^*$ and $y_2^*$ be respectively minima of $f_1$ and $f_2$. If we have $|f_1(y)-f_2(y)|\leq \gamma$ for all $y$, then $f_1(y_2^*)-f_1(y_1^*)\leq 2\gamma$. 
    \end{lem}
    We fix an instance $x$. Let us first introduce some solutions of interest.
    Given a $\bftheta$, let us denote by $\bar y(\bftheta) = (\bar E_1(\bftheta),\bar E_2(\bftheta))$ the result of the prediction problem~\eqref{eq:twoStageSpanningTreeProblem_easy}. Let $\bar z(\bftheta) = (\bar E_1(\bftheta),(\bar E_s(\bftheta)))$ with $\bar E_s(\theta) = \bar E_2(\bftheta)$.
    Let $\hat z(\bftheta) = (\bar E_1(\bftheta),(\hat E_s(\bftheta)))$ where $E_s(\bftheta)$ is the optimal second stage decision for scenario $s$ when the first stage decision is $E_1(\bftheta)$
    \begin{equation}
        \label{eq:optimalSecondStage}
        \hat E_s(\bftheta) \in \argmin \Big\{\sum_{e \in E_s} d_{es} \colon E_s \subseteq E,\,  E_s \cap \bar E_1(\bftheta) = \emptyset, \, (V,E_s \cup \bar E_1(\bftheta)) \in \calT\Big\}.
    \end{equation}
    We denote by $z^{\emptyset} = (\emptyset, (E_s^{\emptyset}))$ the solution with no first stage: $(V,E_s^{\emptyset})$ is a minimum weight spanning tree for second stage weights $(d_{es})_s$ of scenario $s$.
    And finally, we denote by $z(\bftheta) = (E_1,(E_s)_{s \in S})$ the solution returned by our pipeline, which is the solution of minimum cost among $\hat z(\bftheta)$ and $z^{\emptyset}$.
    % $z(\bftheta) = (E_1,(E_s)_{s \in S})$ the solution produced by our pipeline.

    Let $z^* = (E_1^*,(E_{s}^*))$ be an optimal solution of~\eqref{eq:twoStageSpanningTreeProblem}, and $\bar y^{s,*} = (E_1^*,E_{s}^*)$ be the solution of~\eqref{eq:twoStageSpanningTreeProblem_easy} obtained by taking $E_1^*$ as first stage solution and $E_{s}^*$ as second stage solution.

    % Let us also denote by $\bar z(\bftheta) = (\bar E_1(\bftheta),(\bar E_s(\bftheta)))$
    % We denote by $\hat z(\bftheta) = (\hat E_1,(\hat E_s))$ the solution obtained from $\hat y(\bftheta)$ using our decoding pipeline.
    % Let $\bar y(\theta)$ be the solution directly produced by the easy problem: An edge $e$ is in the first stage solution if it is in the first stage solution of the easy problem, and it is in the second stage for all scenario if and only if it is in the (single scenario) second stage of the easy problem.
    % Let $\bar y^{s,*}$ be the solution define by $\bar \bfy_1^{s,*} = \bfy_1^{*}$, and $\bar \bfy_{2,s'}^{s,*} = \bfy_{2,s}^{*}$.
    % Finally, let $y^0$ be the solution with no first stage edges, and optimal second stage edges.

  First, consider solution $\bar z = (\bar E_s,(\bar E_s)_s)$ of~\eqref{eq:twoStageSpanningTreeProblem} such that the second stage solution $\bar E_s$ is identical and equal to  $\bar E_{2}$ for all scenarios $s$ in $S$, and denote by $\bar y = (\bar E_1, \bar E_2)$ the solution of~\eqref{eq:twoStageSpanningTreeProblem_easy} obtained by taking $\bar E_1$ as first stage solution and $\bar E_2$ as second stage solution. It follows from the definition of $\tilde \bfw$ that
    \begin{equation}\label{eq:equality_fh_fe_for_identical_second_stage}
        C(\bar z) = \fe(\bar y, \tilde \bftheta) \quad  \text{where} \quad \tilde \bftheta = f_{\tilde \bfw}(x).
    \end{equation}
    Second, remark that, for any $\bftheta$, $y = (E_1,E_2)$, and $\bfp$ in $\bbR^d(x)$, we have 
    \begin{equation}\label{eq:boundOnPerturbedEasyObjective}
         |\fe(y,\theta) - \fe(y,\theta + Z)| = |\langle Z | y \rangle| = \sum_{e \in E_1}p_{e1}  + \sum_{e \in E_2} p_{e2} \leq \|\bfp\|_1 
    \end{equation}
    where the last inequality comes from the fact that $y$ is the indicator vector of a tree.

    Let $s$ be scenario in $S$. We have 
    \begin{align*}
        C(\hat z(\tilde\bftheta + Z)) &\leq C(\bar z(\tilde\bftheta + Z)) && \text{Optimal second stage} \\
        &= \fe(\bar y(\tilde\bftheta + Z), \tilde\bftheta) && \text{Equation~\eqref{eq:equality_fh_fe_for_identical_second_stage}} \\
        &\leq \fe(\bar y(\tilde\bftheta), \tilde\bftheta)+ 2 \|\bfp\|_1 && \text{Equation~\eqref{eq:boundOnPerturbedEasyObjective} + Lemma~\ref{eq:optimumOfTwoFunctionsWithBoundedDifference}}\\
        &\leq \fe(\bar y^{s,*}, \tilde\bftheta)+ 2 \|\bfp\|_1  && \text{Optimality of $\bar y(\tilde\bftheta)$} \\
        &= \sum_{e \in \bar E_1^*}c_e + \frac{1}{|S|}\sum_{e \in E_s^*}\sum_{s' \in S}d_{es'} + 2 \|\bfp\|_1 \\
        & \leq  \sum_{e \in \bar E_1^*}c_e + \frac{1}{|S|}\sum_{e \in E_s^*}d_{es} + 2 \|\bfp\|_1 && d_{es}' \leq 0 \text{ for all }e,s'.
    \end{align*}
    Furthermore, since $(V,E_s^{\emptyset})$ is a minimum spanning tree with $(d_{es})_e$ edge weights,
    $$
    C(z^{\emptyset}) 
    = \frac{1}{|S|}\sum_{s \in S}\sum_{e \in E_s^{\emptyset}}d_{es} 
    \leq \frac{1}{|S|}\sum_{s \in S}\Big[\sum_{e \in E_1^*}\underbrace{d_{es}}_{\leq 0} + \sum_{e \in E_s^*}d_{es}\Big] 
    \leq  \frac{1}{|S|}\sum_{s \in S}\sum_{e \in E_s^*}d_{es}.$$
    Summing the two previous inequalities, we get
    \begin{align*}
        C(z(\bftheta)) 
        &= \min(C(\hat z(\tilde\bftheta + Z)), C(z^{\emptyset})) \\
        &\leq \frac{|S|C(\hat z(\tilde\bftheta + Z)) + (|S|-1) C(z^{\emptyset})}{2|S|-1} \\
        &\leq \frac{|S|\big(\sum_{e \in E_1^*} c_e+ 2 \|\bfp\|_1\big)  + \sum_s\sum_{e \in E_s^*}d_{es}}{2|S|-1} \\
        &= \frac{|S|  }{2|S|-1} (C(z^*) + 2 \|\bfp\|_1) \leq \frac{1}{2}C(z^*) + \|\bfp\|_1
    \end{align*}
    Since $C(z^*) \leq 0$, we get $C(z(\bftheta)) - C(z^*) \leq \frac{1}{2}|C(z^*)| + \|\bfp\|_1$, which is the result searched.
\end{proof}
\noindent Remark that we have proved the stronger bound $c(y, x) - c^*(x)\leq \frac{|M|-1}{2|M|-1} |c^*(x)| + \|\bfp\|_1 $ when $|S|$ is upper bounded by $M$ on $\calX$.

\paragraph{Objective function approximation}
Let us finally remark that the hypotheses of Theorem~\ref{theo:approximationRatio} are satisfied when $\fe(y,f_{\bfw}(\Gammah)$ is a good approximation of $\bftheta \mapsto g(y,\bftheta)$ for some $\bfw$.
\begin{lem}
    Suppose that  $\bftheta \mapsto g(y,\bftheta)$ is $\kappa_g$ Lipschitz in $\|\cdot\|_{\beta}$, and there exists $\tilde \bfw \in \bfW$
    and $\alpha>0$ is such that, for all $x \in \calX$ and $y \in \calY(x)$, we have
\begin{equation}
    \label{eq:expectedApproximationRatio}
    \frac{|\fh(y,\Gammah) - \fe(y,f_{\tilde \bfw}(\Gammah))|}{u(\Gammah)} \leq \alpha.
\end{equation}
    Then the second hypothesis of Theorem~\ref{theo:approximationRatio} is satisfied with $a= 2\alpha$ and $b = 2\kappa_g$.
\end{lem}
\begin{proof}
    For any $y$ in $\calY(x)$, we have
    \begin{align*}
        &\left|\fh(y,x) - \fe(x,f_{\tilde\bfw}(\Gammah) + \bfp)\right| \\
        &\leq 
        \left|\fh(y,x) - \fe(x,f_{\tilde\bfw}(\Gammah))\right|
        + \left|\fe(x,f_{\tilde\bfw}(\Gammah) - \fe(x,f_{\tilde\bfw}(\Gammah)+ \bfp)\right| \\
        & \leq \alpha u(\Gammah) + \kappa_g \left\|\bfp\right\|_{\beta} 
        % & \leq a u(\Gammah) + \sigma\kappa^{\mathrm{e}}\kappa^{\varphi} \|\bfZ\| \sqrt{|\calE|}.
    \end{align*}
    The results therefore follows from Lemma~\ref{eq:optimumOfTwoFunctionsWithBoundedDifference} and the previous inequality.
\end{proof}

\section{Numerical experiments}
\label{sec:numericExperiments}

This section tests the performance of our algorithms on our running example, the stochastic vehicle scheduling problem, and the $1|r_j|\sum_j C_j$ scheduling problem of Section~\ref{sec:designingAndExamples}.
For the two latter applications, we use the same encoding
$f_{\bfw}$, easy problem solution algorithm, and decoding $\psi$ as in previous contributions \citep{parmentierLearningApproximateIndustrial2019,parmentierScheduling2020}. 
The only difference is that, instead of using the learning by demonstration approaches proposed in these papers, we use the learning by experience approach of this paper.
All the numerical experiments have been performed on a Linux computer running Ubuntu 20.04 with an Intel® Core™ i9-9880H CPU @ 2.30GHz × 16 processor and 64 GiB of memory. 
All the learning problem algorithms are parallelized: The value of the loss on the different instances in the training set are computed in parallel. 
The prediction problem algorithms are not parallelized.

\subsection{Maximum weight two stage spanning tree}
Let us now consider the performance of our pipeline on the maximum weight two stage spanning tree. All the algorithms are implemented in \texttt{julia}. The code to reproduce the numerical experiments is open source\footnote{\url{https://github.com/axelparmentier/MaximumWeightTwoStageSpanningTree.jl}}.

\paragraph{Training, validation and test sets}
We use instances on square grid graphs of width $\{10,\,20,\,30,\,40,\,50,\,60\}$, i.e., with $|V|$ in $\{100, \,400,\,900,\,1600,\,2500,\,3600\}$. 
First stage weig\-hts are uniformly sampled on the integers in $\{-20,\ldots,0\}$. Second stage weights are uniformly sampled on the integers in $\{-K,\ldots,0\}$ with $K \in \{10,\,15,\,20,\,25,\,30\}$. Finally, instances have 5, 10, 15 or 20 second stage scenarios. Our training set, validation set, and test set contain 5 instances for each grid width, weight parameter $K$, and number of scenarios. The training, validation and test sets therefore each contain 600 instances. 

% \begin{wrapfigure}{R}{0.45\textwidth}
%     \begin{center}
%         \includegraphics[width=8cm]{images/gap_to_lag_bound.pdf}
%     \end{center}
%     \caption{Gap with respect to the Lagrangian relaxation bound as a function of (a) the number of vertices $|V|$ with $\varepsilon = 0.001$.  }
%     \label{fig:maxWeightPerfPerVertices}
% \end{wrapfigure}

\paragraph{Bounds and benchmarks}
On each instance of the training, validation, and test set, we solve the Lagrangian relaxation problem using a subgradient descent algorithm for 50,000 iterations, which provides a lower bound on an optimal solution. We also run a Lagrangian heuristic based on the final value of the duals.

We use three benchmarks to evaluate our algorithms: the Lagrangian heuristic, the approximation algorithm of~\citet{escoffierTwostageStochasticMatching2010}, and our pipeline trained by imitation learning using a Fenchel Young loss to reproduce the solution of the Lagrangian heuristic. Remark that since the approximation algorithm, the pipeline learning by imitation, and the pipeline learned with our loss are all instances of our pipeline, they take roughly the same time. On the contrary the Lagrangian heuristic requires to solve the 50,000 iterations of the subgradient descent algorithm, and is therefore 4 order of magnitude slower.

\begin{figure}[!ht]
    % \begin{tabular}{cc}
        \centering
        \begin{minipage}{0.35\textwidth} 
            \centering        
            \begin{tabular}{cc}
                \toprule 
                $\varepsilon$ & Gap \\
            \midrule
            0.0e+00 & 2.8\% \\
            1.0e-04 & 2.7\% \\
            3.0e-04 & 2.7\% \\
            1.0e-03 & 2.7\% \\
            3.0e-03 & 2.7\% \\
            1.0e-02 & 2.7\% \\
            3.0e-02 & 3.9\% \\
            1.0e-01 & 5.5\% \\
            3.0e-01 & 59.5\% \\      
            \bottomrule
        \end{tabular}
        \\ (a)
    \end{minipage}
        % &
        % \vspace{0cm}
        \begin{minipage}{0.6\textwidth} 
            \centering        
        \includegraphics[width=7.5cm]{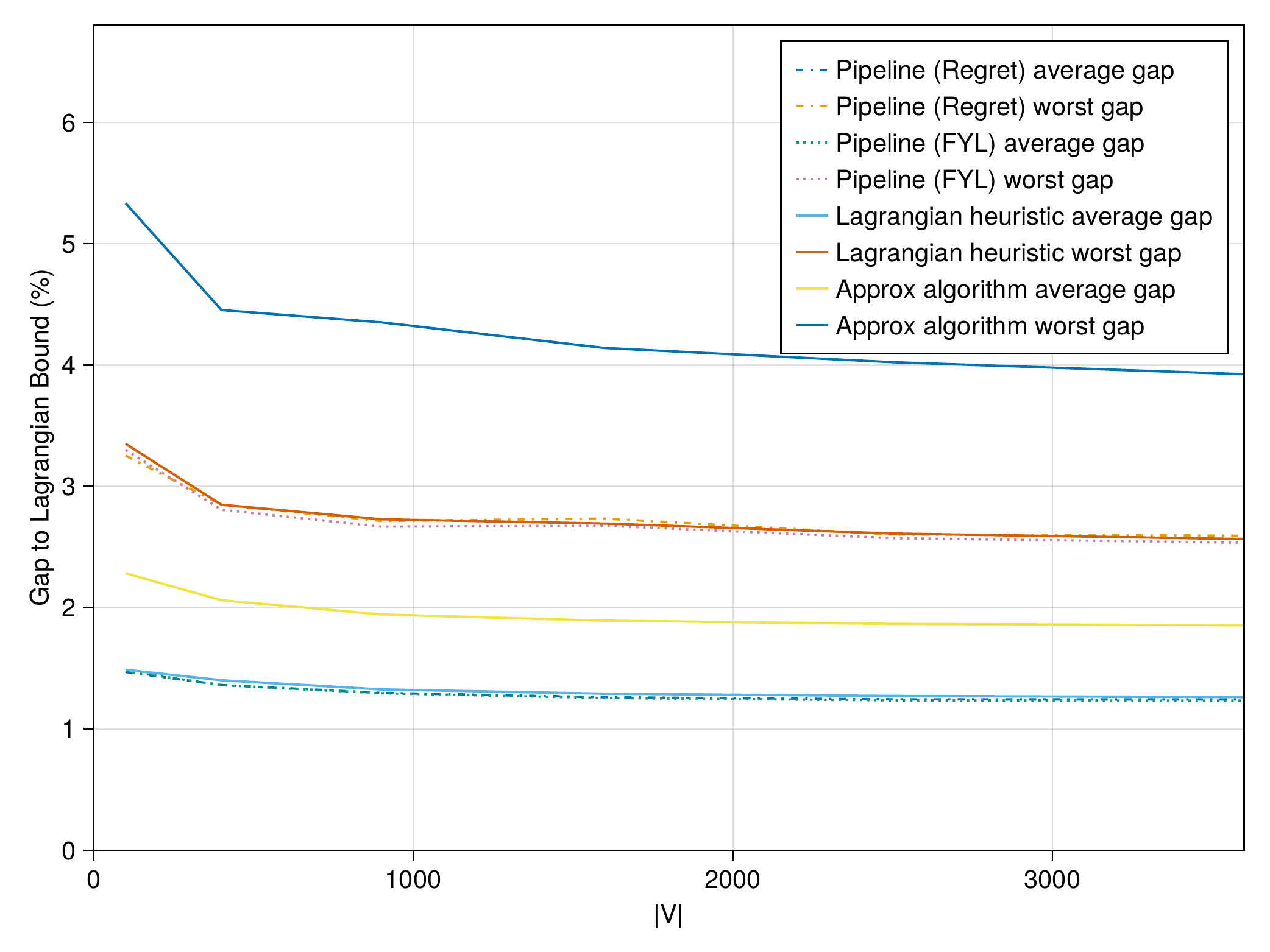} 
        \\ (b)
        \end{minipage}
    %     \\
    %     (a) & (b)
    % \end{tabular}
    % \qquad 
    % (b)
    % \includegraphics[width=8cm]{images/gap_to_lag_bound.pdf}
    \caption{Maximum weight spanning tree. (a) Hyperparameters tuning on the validation dataset for the model learned on the training set. b, Gap with respect to the Lagrangian relaxation bound as a function of (b) the number of vertices $|V|$ with $\varepsilon = 0.001$. }
    \label{fig:MST_hyperparams_perf}
\end{figure}

% \begin{wraptable}{R}{0.35\textwidth}
%     \begin{center}

%     \end{center}
%     \caption{Hyperparameters tuning on large validation dataset for the model learned on the large training set}
%     \label{tab:hyperparameters_tree}
% \end{wraptable}

\paragraph{Hyperparameters tuning}
We use a sample average approximation of our perturbed loss $\ell^{\mathrm{pert}}$ with 20 scenarios. Figure~\ref{fig:MST_hyperparams_perf}.a provides the average value of the gap between the solution returned by our pipeline and the Lagrangian lower bound on the validation set for the model learned with different value of $\varepsilon$. Based on these results, we use $\varepsilon = 0.001$.

\paragraph{Results}
Figure~\ref{fig:MST_hyperparams_perf}.b illustrates the average and worst gap with respect to the Lagrangian bound obtained on the test set for our pipeline and the different benchmarks. The pipeline learning by imitation of by experience enable to match the performance of the Lagrangian heuristic, while being 4 order of magnitude faster. These three algorithms significantly outperform the approximation algorithm.
In summary, our pipeline learned by experience enables to retrieve the performance of the best algorithms as well as the theoretical guarantee of the approximation algorithm.

% \begin{table}
    
%     \caption{Datasets}
% \end{table}

% \begin{table}
    
%     \caption{Performance on small datasets}
% \end{table}

% \begin{table}

%     \caption{Performance on large datasets}
% \end{table}

% Training time (say that the Lagrangian relaxation took roughly one week)

% Maybe two figures giving respectively
% \begin{itemize}
%     \item FYL trained on small dataset (with optimal solution)
%     \item FYL trained on large dataset (with Lagrangian solution)
%     \item Learning by experience on the large dataset (with best )
%     \item Learning by experience on the large dataset with perturbed 
% \end{itemize}

\subsection{Stochastic VSP}

\subsubsection{Setting: Features, post-processing, and instances}
For the numerical experiments on the stochastic VSP, we use the exact same settings as in our previous work \citep{parmentierLearningApproximateIndustrial2019}.
We use the same linear predictor with a vector $\bfphi$ containing 23 features.
And we do not use a post-processing $\psi$.
The easy problem is solved with \texttt{Gurobi 9.0.3} using the LP formulation based on flows.

We also use the same instance generator. 
This generator takes in input the number of tasks $|V|$, the number of scenarios $|\Omega|$ in the sample average approximation, and the seed of the random number generator.
We say that an instance is of moderate size if $|V|\leq 100$, of large size if $100 \leq |V|\leq 750$, and of huge size if $1000\leq |V|$.
Table~\ref{tab:summaryNumericalExperimentsMaster} summarizes the instances generated. 
The first two columns indicate the size of the instances.
A \checkmark{} in the next five columns $|\Omega|$ indicates that instances with $|\Omega|$ scenarios are generated for instance size $|V|$ considered. 
The last five columns detail the composition of the different sets of instances:
Three training sets, one validation set (Val), and a test set (Test).
The table can be read as follows: The training set (small) contains $10 \times |\{50,100,200,500,1000\}| = 50$ instances, each of these having $50$ tasks in $V$, but no larger instances. 
The test set contains instances of all size. For instance, it contains $8 \times |\{50,100,200,500,1000\}| = 40$ instances of size $50$ and $8$ instances of size $5000$.
For the largest sizes, we use only instances with $50$ scenarios for memory reasons: The instances files already weigh several gigabytes.

\begin{table} %{R}{0.7\textwidth}

    \begin{center}
        \scalebox{0.8}{
        \begin{tabular}{lr@{\qquad}ccccc@{\qquad}ccccc}
            \toprule
            \multicolumn{1}{b{0.5cm}}{Size \phantom{200}}
            & \multicolumn{1}{b{0.5cm}}{$|V|$ \phantom{200}}
            & 50 & 100 & \multicolumn{1}{b{0.5cm}}{$|\Omega|$  200} & 500 & 1000
            & \rot{Train (small)}
            & \rot{Train (moderate)}
            & \rot{Train (all)}
            & \rot{Val}
            & \rot{Test}
            \\
            \midrule
            \multirow{3}{*}{Moderate}&50 & \checkmark & \checkmark & \checkmark & \checkmark & \checkmark & 10 & 5 & 1 & 2 & 8 \\
            &75 & \checkmark & \checkmark & \checkmark & \checkmark & \checkmark & & 5 & 1 & 2 & 8 \\
            &100 & \checkmark & \checkmark & \checkmark & \checkmark & \checkmark & & 5 & 1 & 2 & 8 \\
            \midrule
            \multirow{3}{*}{Large}&200 & \checkmark & \checkmark & \checkmark & \checkmark & \checkmark & & 5 & 1 & 2 & 8 \\
            &500 & \checkmark & \checkmark & \checkmark & \checkmark & \checkmark & & & 1 & 2 & 8 \\
            &750 & \checkmark & \checkmark & \checkmark & \checkmark & \checkmark & & & 1 & 2 & 8 \\
            \midrule
            \multirow{3}{*}{Huge}&1000 & \checkmark & \checkmark & \checkmark & \checkmark & \checkmark & & & 1 & 2 & 8 \\
            & 2000 & \checkmark & & & & & & & & 2 & 8 \\
            & 5000 & \checkmark & & & & & & & & 2 & 8 \\
            \bottomrule
        \end{tabular}
        }
    \end{center}
    \caption
    {Instances considered for the stochastic VSP.}
    \label{tab:summaryNumericalExperimentsMaster}

    \end{table}

The small training set, the validation set, and the test set are identical to those previously used~\citep{parmentierLearningApproximateIndustrial2019}.
The validation set, which is used in the learning by demonstration approach, is not used on the learning by experience approach, since we do not optimize on classifiers hyperparameters.
This previous contribution considers only the ``small'' training set, with $50$ instances with $50$ tasks, it uses a learning by demonstration approach and exact solvers cannot handle larger instances.
This is no more a constraint with the learning by experience approach proposed in this paper. We therefore introduce two additional training sets: one that contains 100 instances of moderate size, and one containing 35 instances of all sizes.
These training sets are relatively small in terms of number of instances, but they already lead to significant learning problem computing time and good performance on the test set.

    \subsubsection{Learning algorithm}
    
    On each of the three training sets, we solve the learning problem~\eqref{eq:learningProblem} and the regularized learning problem~\eqref{eq:regularizedLearningProblem}. 
    We use the number of tasks $|V|$ as $u(\Gammah)$. 
    It is not an upper bound on the cost, but the cost of the optimal solution scales almost linearly with $|V|$.
    In both case, we solve the learning problem on the $L_{\infty}$ ball of radius $10$.
    For the regularized learning problem, we use a perturbation strength of intensity $\sigma=1$, and we solve the sample averaged approximation of the problem with $100$ scenarios.
    We evaluate two heuristic algorithms: The DIRECT algorithm~\citep{jonesLipschitzianOptimizationLipschitz1993a} implemented in the \texttt{nlopt} library~\citep{johnsonNLoptNonlinearoptimizationPackage}, and the Bayesian optimization algorithm as it is implemented in the \texttt{bayesopt} library~\citep{martinez-cantinBayesOptBayesianOptimization}. 
    We run each algorithm on $1000$ iterations, which means that they can compute the objective function $1000$ times.
Both algorithms are launched with the default parameters of the libraries.
In particular, the Bayesian optimization algorithm uses the anisotropic kernel with automatic relevance determination \texttt{kSum(kSEARD,kConst)} of the library.

\begin{table} %{R}{0.7\textwidth}

        \begin{center}
            
            \scalebox{0.8}
            {
                \begin{tabular}{llc@{\hspace{0.2cm}}crr@{\hspace{0.2cm}}crr@{\hspace{0.2cm}}c}
                    \toprule
                    \multicolumn{3}{c}{\textbf{Learning problem}} &&
                    \multicolumn{2}{c}{\textbf{DIRECT}} &&
                    \multicolumn{2}{c}{\textbf{Bayes Opt}} 
                    \\
                    \multicolumn{1}{l}{{Obj.}} & 
                    \multicolumn{1}{l}{{Train.~set}} &
                    \multicolumn{1}{l}{pert}&&
                    \multicolumn{1}{c}{CPU time} & \multicolumn{1}{c}{Obj} && CPU time & Obj &\\
                    &&&&(hh:mm:ss)&&&(days, hh:mm:ss)&\\
                    \toprule
                    \multirow{3}{*}{$\ell$}&small&--&&0:01:20&290.39&&0:09:06&287.48\\
                    &moderate&--&&0:10:56&256.39&&0:18:35&259.88\\
                    &all&--&&2:56:40&231.66&&3:56:18&235.04\\
                    \midrule
                    \multirow{3}{*}{$\ell^{\mathrm{pert}}$}&small&100&&0:11:52&286.11&&0:37:22&287.95\\
                    &moderate&100&&1:58:50&258.76&&2:09:17&259.13\\
                    &all&100&&22:44:03&233.84&&1 day, 5:35:30&235.10\\
                    \bottomrule
                    % \% \multicolumn{12}{l}{\strutspace*{-0.4cm}}\
                    % \% \multicolumn{12}{l}{\small Note here}\
                \end{tabular}   
                }             
            \end{center}
        \caption{Performance of the DIRECT \citep{johnsonNLoptNonlinearoptimizationPackage} and Bayesian optimization~\citep{martinez-cantinBayesOptBayesianOptimization} on the learning problems~\eqref{eq:learningProblem} and~\eqref{eq:regularizedLearningProblem} for the stochastic VSP.} 
        \label{tab:vsp-learn-algs}
\end{table}

Table~\ref{tab:vsp-learn-algs} summarizes the result obtained with both algorithms. The first column contains the loss used: $\ell$ for the non-regularized problem~\eqref{eq:learningProblem} and~$\ell^{\mathrm{pert}}$ for the regularized problem.
The next one provides the training set used. 
And the third column provides the number of samples used in the sample average approximation of the perturbation. 
The next four columns give the total computing time for the 1000 iterations and the value of the objective of the learning problem obtained at the end using the DIRECT algorithm and the Bayesian optimization algorithm.

The DIRECT algorithm approximates the value of the function based on a division of the space into hypercubes. 
At each iteration, the function is queried in the most promising hypercube, and the result is used to split the hypercube.
The algorithm leverages a tractable lower bound to identify the most-promising hypercube and the split with few computations. 
Hence, the algorithm is very fast if the function minimized is not computationally intensive.
The Bayesian optimization algorithm builds an approximation of the function minimized: It seeks the best approximation of the function in a reproducing kernel Hilbert space (RKHS) given the data available.
At each iteration, it minimizes an activation function to identify the most promising point according to the model, evaluate the function at that point, and updates the approximation based on the value returned.
Each of these steps are relatively intensive computationally.
Hence, if the function minimized is not computationally intensive, the algorithm will be much slower than the DIRECT algorithm. This is what we observe on the first line of Table~\ref{tab:vsp-learn-algs}.
Furthermore, in \texttt{bayesopt}, the DIRECT algorithm of \texttt{nlopt} is used to minimize the activation function.
Our numerical experiments tend to indicate that the approximation in a RKHS does not enable to find a better solution than the simple exploration with DIRECT after 1000 iterations.
Using Bayesian optimization may however be useful with a smaller iteration budget.
Figure~\ref{fig:compareDirectBayesopt} provides the evolution of the objective function along time for the Bayesian optimization algorithm and the DIRECT algorithm on the learning problem corresponding to the last line of Table~\ref{tab:vsp-learn-algs}.

\begin{figure} %{R}{0.45\textwidth}
    \label{fig:compareDirectBayesopt}
    \begin{center}
        \includegraphics[width=8cm]{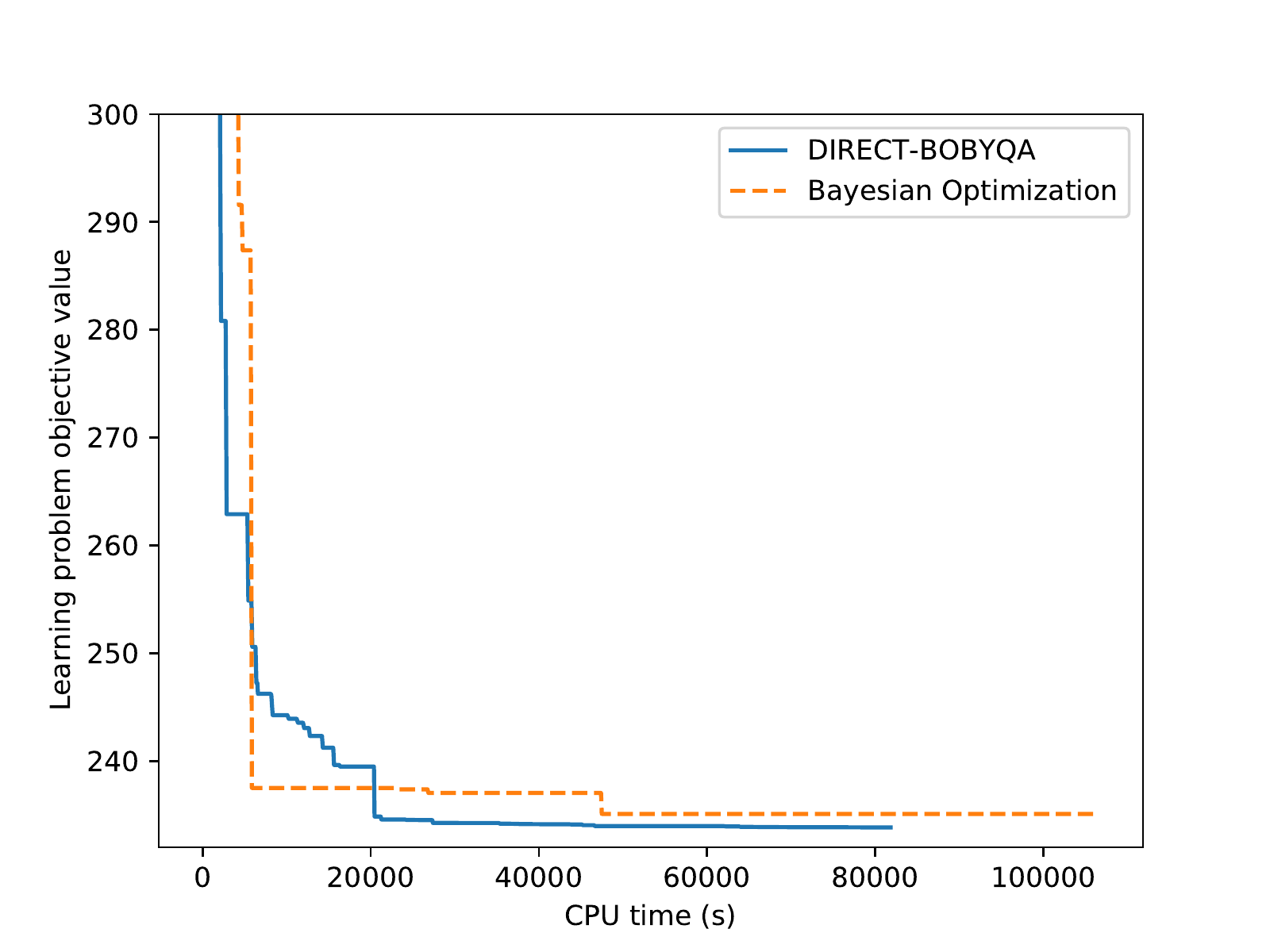}
    \end{center}
    \caption{Learning problem objective value evolution as a function of time on stochastic VSP learning problem with perturbed loss $\ell^{\mathrm{pert}}$ on training set ``all''.}
\end{figure}

Since the DIRECT algorithm gives the best performance on most cases, we keep the $\bfw$ returned by this algorithm for the numerical experiments on the test set. The performance of the DIRECT algorithm could be improved using the optimization on the seed that will be introduced in Section~\ref{ssub:schedulingLearningResults}.

\subsubsection{Algorithm performance on test set}

We now evaluate the performance of our solution pipeline with the $\bfw$ learned.
It has been shown \citep{parmentierLearningApproximateIndustrial2019} that, using  solution pipeline with the $\bfw$ learned by the structured learning approach with 
a conditional random field (CRF) loss on the small training set gives a state-of-the-art algorithm for the problem (the paper uses the maximum likelihood terminology instead of CRF loss).
We therefore use it as a benchmark of the problem.
We have also introduced a new learning by demonstration approach on the problem: We implement the Fenchel Young loss (FYL) structured learning approach \citep{parmentierScheduling2020,berthetLearningDifferentiablePerturbed2020} to obtain a second benchmark.

\begin{table}

    \scalebox{0.57}
    {
        \begin{tabular}{ccc@{\hspace{0.2cm}}crrrr@{\hspace{0.2cm}}crrrr@{\hspace{0.2cm}}crrrr@{\hspace{0.2cm}}crrrr@{\hspace{0.2cm}}c}
            \toprule
            \multicolumn{3}{c}{Learning problem $\bfw$} && 
            \multicolumn{4}{c}{\textbf{Moderate}} &&
            \multicolumn{4}{c}{\textbf{Large}} &&
            \multicolumn{4}{c}{\textbf{Huge}} &&
            \multicolumn{4}{c}{\textbf{All}}  \\
            \multicolumn{1}{c}{Obj} & 
            \multicolumn{1}{c}{Train.~set} &
            \multicolumn{1}{c}{Pert}
            &&
            \multicolumn{1}{c}{$T^{\mathrm{avg}}$} & \multicolumn{1}{c}{$\frac{T^{\mathrm{max}}}{T^{\mathrm{avg}}}$} &\multicolumn{1}{c}{$\delta^{\mathrm{avg}}$} &\multicolumn{1}{c}{$\delta^{\mathrm{max}}$} &&
            \multicolumn{1}{c}{$T^{\mathrm{avg}}$} & \multicolumn{1}{c}{$\frac{T^{\mathrm{max}}}{T^{\mathrm{avg}}}$} &\multicolumn{1}{c}{$\delta^{\mathrm{avg}}$} &\multicolumn{1}{c}{$\delta^{\mathrm{max}}$} &&
            \multicolumn{1}{c}{$T^{\mathrm{avg}}$} & \multicolumn{1}{c}{$\frac{T^{\mathrm{max}}}{T^{\mathrm{avg}}}$} &\multicolumn{1}{c}{$\delta^{\mathrm{avg}}$} &\multicolumn{1}{c}{$\delta^{\mathrm{max}}$} &&
            \multicolumn{1}{c}{$T^{\mathrm{avg}}$} & \multicolumn{1}{c}{$\frac{T^{\mathrm{max}}}{T^{\mathrm{avg}}}$} &\multicolumn{1}{c}{$\delta^{\mathrm{avg}}$} &\multicolumn{1}{c}{$\delta^{\mathrm{max}}$} &\\
            \midrule
            CRF & small & --&&0.03&2.14&9.47\%&20.47\%&&1.21&1.47&2.58\%&6.77\%&&27.69&1.09&1.27\%&1.88\%&&5.74&2.14&5.13\%&20.47\%\\
            FYL & small & 100&&0.03&2.02&1.67\%&4.23\%&&0.97&1.52&0.70\%&2.10\%&&19.20&1.15&0.26\%&1.06\%&&4.04&2.02&1.01\%&4.23\%\\
            \midrule
            \multirow{3}{*}{$\ell$} & small & --&&0.03&2.33&4.37\%&10.35\%&&0.82&1.47&3.65\%&5.56\%&&20.42&1.26&3.29\%&4.90\%&&4.21&2.33&3.88\%&10.35\%\\
             & moderate & --&&0.03&1.77&\textbf{0.31\%}&3.24\%&&0.86&1.39&1.09\%&2.92\%&&17.48&1.13&2.85\%&6.18\%&&3.67&1.77&1.10\%&6.18\%\\
             & all & --&&0.03&1.56&0.52\%&\textbf{1.98\%}&&0.84&1.33&\textbf{0.07\%}&\textbf{0.86\%}&&18.10&1.14&\textbf{0.07\%}&\textbf{0.66\%}&&3.78&1.56&\textbf{0.25\%}&\textbf{1.98\%}\\
            \midrule
            \multirow{3}{*}{$\ell^{\mathrm{pert}}$} & small & 100&&0.03&1.81&2.90\%&6.71\%&&0.84&1.53&2.55\%&4.48\%&&16.29&1.99&2.04\%&3.61\%&&3.44&1.99&2.59\%&6.71\%\\
             & moderate & 100&&0.03&1.98&1.56\%&5.00\%&&0.86&1.42&0.77\%&2.12\%&&16.79&1.09&0.90\%&1.82\%&&3.54&1.98&1.11\%&5.00\%\\
             & all & 100&&0.03&1.63&1.05\%&3.57\%&&0.87&1.45&1.10\%&2.83\%&&18.08&1.30&1.16\%&2.22\%&&3.79&1.63&1.09\%&3.57\%\\
            \bottomrule 
            \multicolumn{24}{l}{\small The best results are in bold. CRF = Conditional Random Field}\
        \end{tabular}
    }

    \caption{Performance of our solution algorithm with different $\bfw$ on the stochastic VSP test set.}
    \label{tab:vsp-alg-perf}
\end{table}

Table~\ref{tab:vsp-alg-perf} summarizes the results obtained.
The first three columns indicate how $\bfw$ has been computed: They provide the loss minimized as objective of the learning problem (Obj), the training set used, and for the approaches that use a perturbation, the number of scenarios used in the sample average approximation (SAA).

The next columns provide the results on the test set.
These columns are divided into four blocks giving results on the subsets moderate, large, huge instances of the test set and on the full test set. 
On each of these subsets of instances, we provide four statistics. 
The statistic $T^{\mathrm{avg}}$ provides the average computing time for our full solution pipeline on the subset of instances considered, which includes the computation of the features and $\varphi_{\bfw}(\Gammah)$, and the resolution of the easy problem with the LP solver (no decoding $\psi$ is used).
Most of this time is spent in the LP solver.
Then, for each instance in the training set, we compute the ratio of the computing time for the instance divided by the average computing time for all the instances of the test set with the same number of tasks $|V|$. Indeed, we expect instances with the same $|V|$ to be of comparable difficulty.
The column $\frac{T^{\mathrm{max}}}{T^{\mathrm{avg}}}$ gives the maximum value of this ratio on the subset of instances considered.
Since we do not have an exact algorithm for the problem, for each instance $\Gammah$ we compute the gap
\begin{equation}\label{eq:gap}
    \frac{c_{\bfw} - c^{\mathrm{best}}}{c^{\mathrm{best}}}
\end{equation}
between the cost $c_{\bfw}$ of the solution returned by our solution pipeline with the $\bfw$ evaluated and the cost of the best solution found for these instances using all the algorithms tested.
The columns $\delta^{\mathrm{avg}}$ and  $\delta^{\mathrm{max}}$ respectively provide the average and the maximum value of this gap on the set of instances considered. 
The two first lines provide the result obtained with the learning by demonstration benchmarks, and the next six ones obtained with the $\bfw$ obtained with the learning algorithms of Table~\ref{tab:vsp-learn-algs}.

We can conclude from these experiments that:
\begin{enumerate}
    \item When using the learning by demonstration approach, the Fenchel Young loss leads to better performances than the conditional random field loss.
    \item Our learning by experience formulation gives slightly weaker performances than the learning by demonstration approach with a Fenchel Young loss when using the same training set.
    \item Our learning by experience approach enables to use a more diversified training set, which enables it to outperform all the previously known approaches. The more diversified the training set, the better the performance.
    \item The regularization by perturbation used tends to decrease the performance of the algorithm. This statement may no longer hold if we optimized the strength of the perturbation using a validation set.
\end{enumerate}
% The last .

\subsection{Single machine scheduling problem $1|r_j|\sum_j C_j$}

\subsubsection{Setting}
\label{ssub:schedulingSetting}

We use the exact same setting as the previous contribution on this problem~\citep{parmentierScheduling2020}. 
In particular, that paper introduces a vector of 66 features, and a subset of 27 features that leads to better performances.
With the objective of testing what our learning algorithm can do on a larger dimensional problem, we focus ourselves on the problem with 66 features.
And we also use the four kinds of decoding algorithms in that paper: no decoding (no $\psi$), a local search (LS), the same local search followed by release date improvement (RDI) algorithm (RDI $\circ$ LS), and the perturbed versions of the last algorithm, (pert RDI $\circ$ LS) where the solution pipeline is applied with $\bfw + \bfZ$ for 150 different samples of a standard Gaussian $\bfZ$, and keep the best solution found. RDI is a classic heuristic for scheduling problems, which is more time consuming but more efficient than the local search.

% \subsubsection{Instances}
We use the same generator of instances as previous contributions \citep{DT02,parmentierScheduling2020}. For a given instance with $n$ jobs, processing times $p_j$ are drawn at random following the uniform distribution $[1;100]$ and release dates $r_j$ are drawn at random following the uniform distribution $[1;50.5 \, n  \rho]$. Parameter $\rho$ enables to generate instances of different difficulties: We consider $\rho\in\{0.2, 0.4, 0.6, 0.8, 1.0,\allowbreak 1.25, 1.5, 1.75, 2.0, 3.0\}$. For each value of $n$ and $\rho$, $N$ instances are randomly generated leading for a fixed value of $n$ to $10N$ instances. 
For the learning by demonstration approach, we use the same training set as~\citep{parmentierScheduling2020} with $n \in \{50, 70, 90, 110\}$ and $N = 100$, leading to a total of $4000$ instances.
We do not use larger instances because we do not have access to optimal solutions for larger instances.
For the learning by experience approach, we use $n \in \calN := \{50,75,100,150,200,300,500,750,1000,1500,2000,3000\}$ and $N=20$, leading to $2400$ instances. 
In the test set, we use a distinct set of $2400$ instances with $n \in\calN$ and $N=20$. This test set is almost identical to the one used in the literature~\citep{parmentierScheduling2020}, the only difference being that instances with $n = 2500$ have been replaced by instance with $n=3000$ to get a more balanced test set.
Table~\ref{tab:schedulingTestSet} shows how we have partitioned this test set by number of jobs $n$ in the instances, to get sets of instances of moderate, large, and huge size.

\begin{table}

        \begin{center}
            
            \scalebox{0.9}{
                \begin{tabular}{lccc}
                    \toprule
                    &\multicolumn{3}{c}{Subsets of instances}
                    \\
                    & Moderate
                    & Large
                    & Huge
                    \\
                    \midrule
                    Size $n$ of instances in subset
                    & $\{50,75,100,150\}$
                    & $\{200,300,500,750\}$
                    & $\{1000,1500,2000,3000\}$
                    \\
                    \bottomrule
                \end{tabular}
                }
            \end{center}

            \caption{Size of the $1|r_j|\sum_j C_j$ instances in the subsets of the test set.}
            \label{tab:schedulingTestSet}
            \end{table}

    \subsubsection{Learning algorithm}
\label{ssub:schedulingLearningResults}

\begin{table} %{R}{0.6\textwidth}

    \begin{center}
        
        \scalebox{0.8}{
            \begin{tabular}{cccc@{\hspace{0.3cm}}rrr}
                \toprule
                %     \multicolumn{4}{c}{Learning problem $\bfw$}&&
                %     \multicolumn{11}{c}{Test set results (with several $\psi$)} \\
                %     &&&&&
                %     \multicolumn{2}{c}{no $\psi$} &&
                %     \multicolumn{2}{c}{LS} &&
                %     \multicolumn{2}{c}{RDI $\circ$ LS} &&
                %     % \multicolumn{2}{c}{LS pert} &&
                %     \multicolumn{2}{c}{pert RDI $\circ$ LS} 
                % \\
                obj&
                iter&
                pert&
                $\psi$ &
                \multicolumn{1}{c}{Tot.~CPU} &
                Avg $\hat L$ &
                Best $\hat L$
                % \\
                % &&&&
                % \multicolumn{1}{c}{(days, hh:mm:ss)} &
                \\
                \toprule
                $\ell$ & 1000 & -- & -- & 0:05:46 & 36.71 & 35.19 \\
                $\ell$ & 2500 & -- & -- & 0:12:31 & 36.64 & 35.29 \\
                $\ell^{\mathrm{pert}}$ & 1000 & 100 & -- & 9:09:06 & 36.76 & 35.27 \\
                $\ell^{\mathrm{pert}}$ & 2500 & 100 & -- & 20:45:52 & 36.69 & 35.25 \\
                $\ell$ & 1000 & -- & LS & 0:52:54 & 35.13 & 35.06 \\
                $\ell$ & 2500 & -- & LS & 1:47:42 & 35.12 & 35.06 \\
                $\ell^{\mathrm{pert}}$ & 1000 & 100 & LS & 3 days, 15:14:42 & 35.12 & 35.06 \\
                $\ell^{\mathrm{pert}}$ & 2500 & 100 & LS & 7 days, 15:24:08 & 35.11 & 35.06 \\        
                \bottomrule
                \multicolumn{7}{l}{\small Tot.~CPU is given in days, hh:mm:ss.}
            \end{tabular}
            }
        \end{center}
        \caption{Learning algorithm results on $1|r_j|\sum_j C_j$. }
        \label{tab:learningScheduling}
        \end{table}

We use $n(n+1)$ as $u(\Gammah)$. 
It is not an upper bound on the cost, but the cost of the optimal solution scales roughly linearly with $u(\Gammah)$.
We draw lessons from the stochastic VSP and use only a diverse training set of $4000$ instances of all size in the training set.
And because our solution pipeline for $1|r_j|\sum_jC_j$ is much faster than the one for the stochastic vehicle scheduling problem, we can use a larger training set.
% This much larger training set is made possible by the fact that our solution pipeline for the $1|r_j|\sum_j C_j$ problem than the one of the stochastic VSP due to a simpler easy problem.
And we introduce two new perspectives.
First, the DIRECT algorithm uses a random number generator. We observed that its performance is very dependent on the seed of the random number generator, and that using a larger number of iterations does not necessarily compensate for the poor performance that would come from a bad seed.
We therefore launch the algorithm with 10 different seeds, each time with a 1000 iterations budget, and report the best result.
Second, as underlined in Section~\ref{sub:practicalRemarks}, it can be natural to use the loss $\ell^{\psi}$ where, instead of using the output of the easy problem, we use the output of the post-processing $\psi$.
In our case, the post-processing is in two steps: first the local search, second the RDI heuristic. Since RDI is time-consuming, using it would lead to very large computing times on the training set used. 
We therefore take the solution at the end of the local search.
%  the DIRECT algorithm to min

        Table~\ref{tab:learningScheduling} summarizes the results obtained.
        The first column indicate if the perturbed loss or the non-perturbed loss has been used. The second indicates the number of iterations of DIRECT used. The third column indicates the number of scenarios used in the sample average approximation when the perturbed loss is used. And ``--'' (resp.~LS) in the fourth column indicates if no (resp~the local search) post-processing has been applied to the solution used in the loss.
        The column Tot.~CPU then provides the total CPU time of the 10 runs of DIRECT with different seeds. Finally, the columns Avg $\hat L$ and Best $\hat L$ give respectively the average and the best loss value of the best solution found by DIRECT algorithm on the 10 seeds used.
        
        We can conclude from these results that optimizing on the seed seems a good idea. We also observe that the loss function after the local search is smaller, which is natural given that the local search improves the solution found by the easy problem.

        \subsubsection{Algorithm performance on test set}
        
\begin{table}

    \scalebox{0.9}{
    \begin{tabular}{cccc@{\hspace{0.3cm}}crr@{\hspace{0.2cm}}crr@{\hspace{0.2cm}}crr@{\hspace{0.2cm}}crr}
    \toprule
        \multicolumn{4}{c}{Learning problem $\bfw$}&&
        \multicolumn{11}{c}{Test set results (with several $\psi$)} \\
        &&&&&
        \multicolumn{2}{c}{no $\psi$} &&
        \multicolumn{2}{c}{LS} &&
        \multicolumn{2}{c}{RDI $\circ$ LS} &&
        % \multicolumn{2}{c}{LS pert} &&
        \multicolumn{2}{c}{pert RDI $\circ$ LS}  
    \\
        obj&
        iter&
        pert&
        $\psi$ &&
        \multicolumn{1}{c}{$\delta^{\mathrm{avg}}$} &\multicolumn{1}{c}{$\delta^{\mathrm{max}}$} &&
        \multicolumn{1}{c}{$\delta^{\mathrm{avg}}$} &\multicolumn{1}{c}{$\delta^{\mathrm{max}}$} &&
        \multicolumn{1}{c}{$\delta^{\mathrm{avg}}$} &\multicolumn{1}{c}{$\delta^{\mathrm{max}}$} &&
        \multicolumn{1}{c}{$\delta^{\mathrm{avg}}$} &\multicolumn{1}{c}{$\delta^{\mathrm{max}}$} 
    \\
    \toprule
    FYL & -- & -- & --&& 1.81\%& 8.57\%&& 1.10\%& 6.88\%&& 0.07\%& 3.41\%&& 0.02\%& 0.46\%\\
    % fyl-27 & -- & -- & --&& 1.47\%& 9.36\%&& 0.95\%& 6.68\%&& 0.07\%& 2.05\%&& 0.01\%& 0.43\%\\
    \midrule
    $\ell$ & 1000 & -- & --&& 0.63\%& 24.53\%&& 0.34\%& 4.59\%&& 0.06\%& 1.65\%&& 0.02\%& 1.53\%\\
    $\ell$ & 2500 & -- & --&& 1.08\%& 21.19\%&& 0.30\%& 6.33\%&& 0.07\%& 1.71\%&& 0.04\%& 1.71\%\\
    $\ell^{\mathrm{pert}}$ & 1000 & 100 & --&& 0.83\%& 23.61\%&& 0.38\%& 3.64\%&& 0.06\%& 1.65\%&& 0.03\%& 1.37\%\\
    $\ell^{\mathrm{pert}}$ & 2500 & 100 & --&& 0.75\%& 19.54\%&& 0.33\%& 3.98\%&& 0.06\%& 1.69\%&& 0.02\%& 1.37\%\\
    \midrule
    $\ell$ & 1000 & -- & LS&& 10.51\%& 54.67\%&& 0.02\%& 1.30\%&& 0.01\%& 1.12\%&& 0.01\%& 1.12\%\\
    $\ell$ & 2500 & -- & LS&& 10.16\%& 55.70\%&& 0.02\%& 1.30\%&& 0.01\%& 1.12\%&& 0.01\%& 1.12\%\\
    $\ell^{\mathrm{pert}}$ & 1000 & 100 & LS&& 10.54\%& 55.22\%&& 0.03\%& 2.26\%&& 0.02\%& 2.26\%&& 0.02\%& 2.26\%\\
    $\ell^{\mathrm{pert}}$ & 2500 & 100 & LS&& 10.51\%& 53.64\%&& 0.03\%& 2.26\%&& 0.02\%& 2.26\%&& 0.02\%& 2.26\%\\
    \bottomrule
    \end{tabular}
    }
    \caption{Performance of our solution algorithms with different $\bfw$ on the $1|r_j|\sum_j C_j$ test set.}
    \label{tab:schedulingFullTestSet}
\end{table}

Table~\ref{tab:schedulingFullTestSet} summarizes the results obtained with the different $\bfw$ on the full test set.
The first line corresponds to the Fenchel young loss (FYL) of the learning by demonstration approach previously proposed~\citep{parmentierScheduling2020}, and serves as a benchmark. 
The next eight ones correspond to the parameters obtained solving the learning problem described in this paper with the settings of Table~\ref{tab:learningScheduling}.
The first four columns describe the parameters of the learning problems used to obtained $\bfw$ and are identical to those of Table~\ref{tab:learningScheduling}.
The next columns indicate the average results on the full test set for the four kind of post-processing described in Section~\ref{ssub:schedulingSetting}.
Again, we provide the average $\delta^{\mathrm{avg}}$ and the worse $\delta^{max}$ values of the gap~\eqref{eq:gap} between the solution found by the algorithm and the best solution found by all the algorithms.

Two conclusions can be drawn from these results:
\begin{enumerate}
    \item The solution obtained with our loss by experience approach tend to outperform on average those obtained using the Fenchel Young loss, but tend to have a poorer worst case behavior.
    \item Using the loss with post-processing tend to improve the performance on the test set with the pipelines that use this preprocessing, and possibly other after. But it decreases the performance on the pipeline which do not use it.
\end{enumerate}

\begin{table}

    \scalebox{0.75}
    {
        \begin{tabular}{l@{\hspace{0.2cm}}ccccc@{\hspace{0.2cm}}crrrr@{\hspace{0.2cm}}crrrr@{\hspace{0.2cm}}crrrr}
            \toprule
            Pred. &&
            \multicolumn{4}{c}{$\bfw$} && 
            \multicolumn{3}{c}{\textbf{Moderate}} &&
            \multicolumn{3}{c}{\textbf{Large}} &&
            \multicolumn{3}{c}{\textbf{Huge}} \\
            $\psi$
            &&
            obj&
            iter&
            $|\Omega|$&
            $\psi$ 
            &&
            \multicolumn{1}{c}{$T^{\mathrm{avg}}$} & \multicolumn{1}{c}{$\delta^{\mathrm{avg}}$} &\multicolumn{1}{c}{$\delta^{\mathrm{max}}$} 
            &&
            \multicolumn{1}{c}{$T^{\mathrm{avg}}$} & \multicolumn{1}{c}{$\delta^{\mathrm{avg}}$} &\multicolumn{1}{c}{$\delta^{\mathrm{max}}$} 
            &&
            \multicolumn{1}{c}{$T^{\mathrm{avg}}$} & \multicolumn{1}{c}{$\delta^{\mathrm{avg}}$} &\multicolumn{1}{c}{$\delta^{\mathrm{max}}$} 
            \\
            \midrule

            \multirow{9}{*}{no $\psi$} && FYL & -- & -- & --&&0.01&1.13\%&8.57\%&&0.40&1.80\%&6.06\%&&102.82&2.50\%&6.47\%\\
            && $\ell$ & 1000 & -- & --&&0.01&1.38\%&24.53\%&&0.20&0.39\%&2.33\%&&25.03&0.11\%&0.54\%\\
            && $\ell$ & 2500 & -- & --&&0.01&2.59\%&21.19\%&&0.15&0.52\%&3.61\%&&22.85&0.14\%&0.96\%\\
            && $\ell^{\mathrm{pert}}$ & 1000 & 100 & --&&0.01&1.51\%&23.61\%&&0.25&0.66\%&4.27\%&&31.95&0.33\%&2.23\%\\
            && $\ell^{\mathrm{pert}}$ & 2500 & 100 & --&&0.01&1.25\%&19.54\%&&0.18&0.59\%&3.46\%&&15.91&0.41\%&2.00\%\\
            && $\ell$ & 1000 & -- & LS&&0.01&10.19\%&54.67\%&&0.07&10.64\%&51.23\%&&2.04&10.70\%&46.99\%\\
            && $\ell$ & 2500 & -- & LS&&0.01&10.31\%&55.70\%&&0.07&10.24\%&50.63\%&&2.20&9.93\%&43.87\%\\
            && $\ell^{\mathrm{pert}}$ & 1000 & 100 & LS&&0.01&10.28\%&55.22\%&&0.07&10.64\%&49.59\%&&2.47&10.70\%&46.04\%\\
            && $\ell^{\mathrm{pert}}$ & 2500 & 100 & LS&&0.01&10.01\%&53.64\%&&0.07&10.63\%&49.27\%&&2.45&10.90\%&46.36\%\\
            \midrule
            \multirow{9}{*}{\shortstack{pert \\ RDI $\circ$ LS}}  && FYL & -- & -- & --&&0.36&0.02\%&0.46\%&&2.58&0.02\%&0.24\%&&208.72&0.02\%&0.16\%\\
            && $\ell$ & 1000 & -- & --&&0.48&0.05\%&1.53\%&&2.49&0.02\%&0.62\%&&50.98&0.00\%&0.10\%\\
            && $\ell$ & 2500 & -- & --&&0.50&0.09\%&1.71\%&&2.44&0.02\%&0.34\%&&45.37&0.00\%&0.10\%\\
            && $\ell^{\mathrm{pert}}$ & 1000 & 100 & --&&0.49&0.05\%&1.37\%&&2.61&0.02\%&0.66\%&&65.53&0.00\%&0.11\%\\
            && $\ell^{\mathrm{pert}}$ & 2500 & 100 & --&&0.48&0.05\%&1.37\%&&2.54&0.02\%&0.61\%&&40.28&0.00\%&0.10\%\\
            && $\ell$ & 1000 & -- & LS&&0.48&0.03\%&1.12\%&&2.34&0.01\%&0.27\%&&14.06&0.00\%&0.02\%\\
            && $\ell$ & 2500 & -- & LS&&0.49&0.03\%&1.12\%&&2.32&0.00\%&0.14\%&&14.19&0.00\%&0.01\%\\
            && $\ell^{\mathrm{pert}}$ & 1000 & 100 & LS&&0.48&0.05\%&2.26\%&&2.36&0.01\%&0.27\%&&14.44&0.00\%&0.05\%\\
            && $\ell^{\mathrm{pert}}$ & 2500 & 100 & LS&&0.49&0.05\%&2.26\%&&2.39&0.01\%&0.27\%&&14.59&0.00\%&0.05\%\\           
 \bottomrule
 \multicolumn{12}{l}{\small $T^{\mathrm{avg}}$ is given in seconds.}
        \end{tabular} 
    }
    \caption{Influence of instances size on the performance of our solution algorithms with different $\bfw$ on the $1|r_j|\sum_j C_j$ test set.}
    \label{tab:schedulingResultsDetailed}
\end{table}

Finally, Table~\ref{tab:schedulingResultsDetailed} details the results for the fastest  (no $\psi$) and the most accurate one (pert RDI $\circ$ LS) pipeline on the subsets of instances of moderate, large, and huge size.
In addition to the gaps, the average computing time $T^{\mathrm{avg}}$ is provided.
Again, we can observe that:
\begin{enumerate}[resume]
    \item Because the learning by experience approach enables to use a diversified set of instances in the training set, it outperforms the learning by demonstration approach on large and huge instances.
\end{enumerate}

\section{Conlusion}

We have focused on heuristic algorithms for hard combinatorial optimization problems based on machine learning pipelines with a simpler combinatorial optimization problem as layer.
Previous contributions in the literature required training sets with instances and their optimal solutions to train such pipelines. 
We have shown that the solutions are not necessarily needed, and we can learn such pipelines by experience if we formulate the learning problem as a regret minimization problems.
This widens the potential applications of such methods since it removes the need of an alternative algorithm for the hard problem to build the training set.
Furthermore, even when such an algorithm exits, it may not be able to handle large instances.
The learning by experience approach can therefore use larger instances in its training set, and can take into account the effect of potential post-processings. 
These two ingredients enable to scale better on large instances.
Finally, we have shown that, if an approximation algorithm can be encoded in the pipeline with a given parametrization, then the parametrization learned by experience retains the approximation guarantee while giving a more efficient algorithm in practice.

Future contributions may focus on providing richer statistical models in the neural network, which would require to adapt the learning algorithm. Furthermore, the approximation ratio guarantee could be extended to more general settings.

\subsection*{Acknowledgements}
I am grateful to Yohann de Castro and Julien Reygnier for their help on Section~\ref{sec:convergenceEstimator}, 
%to Jean-François Delmas for the proof of Lemma~\ref{lem:boundedPerturbedLipchitz}, 
and to Vincent T'Kindt for his help on the scheduling problem.

\bibliographystyle{plainnat}
\bibliography{ml4or}

\appendix

\newpage
\section{Proof of Theorem~\ref{theo:perturbedBoundedLossLearningRate}}
\label{sec:ProofOfTheoremtheo:perturbedBoundedLossLearningRate}

\subsection{Background on Rademacher complexity and metric entropy method}
    This section introduces some classical tools of statistical learning theory~\citep{bousquetIntroductionStatisticalLearning2004}. The lecture notes of \citep{wolfMathematicalFoundationsSupervised2018} contain detailed proofs.

    We place ourselves in the setting of Section~\ref{sub:backgroundStatisticalLearning}.
    Let $\calF$ be the family of functions $\Big\{\xi \mapsto \ell(\xi,\bfw)\colon \bfw \in \bfW\Big\}$. The \emph{Rademacher complexity of $\calF$ is}
    $$\calR_n(\calF) = \bbE_{\xi_i,\sigma_i}\biggl[\sup_{\bfw \in \bfW}\frac{1}{n}\sum_{i=1}^n \sigma_i \ell(\xi_i,\bfw)\biggr] $$
    where the $\sigma_i$ are i.i.d.~Rademacher variables, i.e., variables equal to $1$ with probability $1/2$, and to $-1$ otherwise.
The following well-known result bounds the excess risk based on the Rademacher complexity.
\begin{prop}\label{prop:rademacherBoundsExcessRisk}
    With probability at least $1-\delta$, we have
    $$ L(\hat \bfw_n) - L(\bfw^*) \leq 4R_n(\calF) + \sqrt{\frac{2\log(2/\delta)}{n}}. $$
\end{prop}

The metric entropy method enables to bound the Rademacher complexity.
The \emph{empirical Rademacher complexity} of $\calF$ is obtained when we replace the expectation over $\xi_i$ by its values for the training set used $\xi_1,\ldots,\xi_n$.
$$\hat \calR_n(\calF) = \bbE\biggl[\sup_{\bfw \in \bfW}\frac{1}{n}\sum_{i=1}^n \sigma_i \ell(\xi_i,\bfw) | \xi_1,\ldots,\xi_n \biggr]  $$
and we have $\calR_n(\calF) = \bbE[\hat \calR_n(\calF)]$.

Given $n$ instances $\xi_1,\ldots,\xi_2$ and the corresponding distribution $\hat \mu_n$ on $\Xi$, the pseudometric $L_2(\hat \mu_n)$ on $\calF$ is the $L_2$ norm induced by $\hat \mu_n$ on $\calF$
$$ \left\|\ell(\cdot,\bfw) - \ell(\cdot,\bfw')\right\|_{2,\hat \mu_n} = \sqrt{\frac{1}{n}\sum_{i=1}^n\left(\ell(\xi_i,\bfw) - \ell(\xi_i,\bfw')\right)^2}$$
We denote by $B_{\varepsilon, L_2(\hat \mu_n)}(\ell(\cdot,\bfw))$ the ball of radius $\varepsilon$ centered in $\ell(\cdot,\bfw)$.
The set covering number of $\calF$ with respect to $L_2(\hat \mu_n)$ is 
$$N(\varepsilon,\calF, L_2(\hat \mu_n)) = \min \Big\{m\colon\exists\{\bfw_1,\bfw_m\}\subseteq \bbR^d, \calF \subseteq \bigcup_{j=1}^m B_{\varepsilon, L_2(\hat \mu_n)}(\ell(\cdot,\bfw_m)) \Big\}.$$
The following result bounds the empirical Rademacher complexity from the covering number.
\begin{prop}\label{prop:Dudley}(Dudley's theorem)
    Let $\calF$ be a family of mapping from $\calZ$ to $[-1,1]$, then
    $$\hat \calR_n(\calF) \leq 12 \int_{0}^{\infty} \sqrt{\frac{\log N(\varepsilon,\calF, L_2(\hat \mu_n))}{n}} d\varepsilon $$
\end{prop}

\subsection{Proof of Theorem~\ref{theo:perturbedBoundedLossLearningRate}}

The proof is as follows.
We show that the Gaussian perturbation turns any bounded function in a Lipschitz function.
Hence, the perturbed loss is Lipschitz. 
This implies an upper bound on the covering number, and Dudley's theorem enables to conclude.

Let $\bfZ$ be a centered standard Gaussian vector on $\bbR^d$. 
It is well known that
    \begin{equation}\label{eq:upperBoundExpectationNormGaussian}
        \bbE(\|\bfZ\|) \leq \sqrt{d}.
    \end{equation}
Indeed, applying $u \leq (1+u^2)/2$ with $u = \sqrt{\frac{1}{d}\sum Z_i^2}$ gives $\frac{1}{\sqrt{d}} \|\bfZ\| \leq \frac{1}{2}(1 + \frac{1}{d}\sum_{i=1}^dZ_i^2)$. Taking the expectation and using $\bbE(Z_i^2) = 1$ gives~\eqref{eq:upperBoundExpectationNormGaussian}.

\begin{lem}\label{lem:boundedPerturbedLipchitz}
    Let $g : \bbR^d \rightarrow [0,1]$ be an integrable function, $\bfZ$ a standard normal random vector on $\bbR^d$, $\sigma>0$ a positive real number, and $G(\bfw) = \bbE g(\bfw + \sigma Z)$. Then $\bfw \mapsto G(\bfw)$ is $\frac{\sqrt{d}}{\sigma}$-Lipchitz.
\end{lem}
\begin{proof}
    Let $h$ be the density of $\tilde \bfZ = \sigma\bfZ$. We have
    $$ G(\bfw) = \int h(\bfz)g(\bfz+\bfw) = \int h(\bfz-\bfw)g(\bfz) $$
    By dominated convergence, we have
    $$ \nabla G(\bfw) = - \int \nabla h(\bfz-\bfw)g(\bfz) = - \int \nabla h(\bfz)g(\bfz + \bfw)  $$
    From there, using the facts that $|g(\bfw)| \leq 1$ and $\bfZ$ is a standard Gaussian, we get
    $$ \|\nabla G(\bfw)\| \leq \int \|\nabla h(\bfz)\| = \int \|\frac{\bfz }{\sigma^2(\sqrt{2\pi\sigma})^n}e^{-\frac{\|\bfz\|^2}{2\sigma^2}}\| = \frac{\bbE(\|\tilde\bfZ\|) }{\sigma^2}
    = \frac{1}{\sigma} \bbE(\|\bfZ\|) \leq \frac{\sqrt{d}}{\sigma}
    % \leq d \frac{\bbE(|Y_1|) }{\sigma^2} = \frac{d}{\sigma}\sqrt{\frac{2}{\pi}},
    $$
    which gives the result. 
\end{proof}
Given an arbitrary element $\xi$ in $\Xi$,
Lemma~\ref{lem:boundedPerturbedLipchitz} applied with $g = \ell(\xi,\cdot)$ gives
$$ |\ell(\xi,\bfw) - \ell(\xi,\bfw')| \leq \frac{\sqrt{d}}{\sigma} \|\bfw - \bfw'\|_2 $$
Hence
$$  \left\|\ell(\cdot,\bfw) - \ell(\cdot,\bfw')\right\|_{2,\hat \mu_n} =  \sqrt{\frac{1}{n}\sum_{i=1}^n\left(\ell(\xi_i,\bfw) - \ell(\xi_i,\bfw')\right)^2} \leq \frac{\sqrt{d}}{\sigma} \|\bfw - \bfw'\|.$$
As a consequence, if $\bfw_1,\ldots,\bfw_m$ is an $\frac{\varepsilon\sigma}{\sqrt{d}}$ covering of $\bfW$ endowed with the Euclidean norm, then $\ell(\cdot,\bfw_1),\ldots,\ell(\cdot,\bfw_m)$ is an $\varepsilon$ covering of $\calF$.
Hence, if $\bfW$ is contained in the Euclidean ball of radius $M$, we get
$$N(\varepsilon,\calF,L_2(\hat \mu_n)) \leq N(\frac{\varepsilon\sigma}{\sqrt{d}}, \bfW = \bbB^d(M),\|\cdot\|_2) \leq \left(\frac{M\sqrt{d}}{\varepsilon\sigma}\right)^d $$
for $\varepsilon \leq \frac{M\sqrt{d}}{\sigma}$ and $ N(\varepsilon,\calF,L_2(\hat \mu_n)) = 1$ otherwise.
And we obtain
$$ \log N(\varepsilon,\calF,L_2(\hat \mu_n)) \leq d \bigl(\log(M\sqrt{d}/\sigma) - \log{\varepsilon })\bigr)$$
for $\varepsilon \leq \frac{M\sqrt{d}}{ \sigma}$ and $ \log N(\varepsilon,\calF,L_2(\hat \mu_n))= 0$ otherwise.

Proposition~\ref{prop:Dudley} then gives
\begin{align*}
    \hat R_n(\calF) \leq 12\int_{0}^{\frac{M\sqrt{d}}{\sigma}} \sqrt{d\frac{\log(\frac{M\sqrt{d}}{\sigma}) - \log \varepsilon}{n}} d\varepsilon 
    = 12 \sqrt{\frac{d}{n}} \int_{0}^{\frac{M\sqrt{d}}{\sigma}} \sqrt{-\log \Big( \frac{\varepsilon}{M\sqrt{d}/\sigma} \Big)}d\varepsilon
    =\frac{C}{4}\frac{Md}{\sigma\sqrt{n}} 
\end{align*}
with $C =48\int_0^1\sqrt{-\log{x}}dx$.
Remark that the bound on $\hat R_n(\calF)$ we obtain does not depend on the sample $\xi_1,\ldots,\xi_n$, and is therefore also valid for $R_n(\calF) = \bbE\big(\hat R_n(\calF)\big)$.
Proposition~\ref{prop:rademacherBoundsExcessRisk} then gives Theorem~\ref{theo:perturbedBoundedLossLearningRate}.

\end{document}